%% file: arxiv.tex
\documentclass{article}

\PassOptionsToPackage{numbers, sort&compress}{natbib}

\usepackage[preprint]{neurips_2026}

\usepackage[utf8]{inputenc} 
\usepackage[breaklinks,colorlinks,allcolors=ratioblue]{hyperref}
\usepackage{amsfonts}       
\usepackage{nicefrac}       
\usepackage{microtype}      
\usepackage{xcolor}         

\usepackage{url}
\usepackage{graphicx}
\usepackage{float}
\usepackage{placeins}
\usepackage{afterpage}
\usepackage{wrapfig}

\usepackage{array}
\usepackage{tabularx}
\usepackage{booktabs}
\usepackage[table]{xcolor}
\usepackage{multirow}
\usepackage{adjustbox}
\usepackage{makecell}
\usepackage{graphicx}

\usepackage{amsmath}
\usepackage{amssymb}
\usepackage{amsthm}
\usepackage{bm}

\usepackage{booktabs}      
\usepackage[table,svgnames]{xcolor}   
\usepackage{multirow}      
\usepackage{subcaption}
\usepackage{adjustbox}
\usepackage{tabularx}
\usepackage{makecell}

\usepackage{pifont}
\usepackage[ruled,vlined]{algorithm2e}
\usepackage[most]{tcolorbox}
\usepackage[inline]{enumitem}
\usepackage[T1]{fontenc}
\usepackage{fontawesome}
\usepackage{titletoc}

\newtheorem{definition}{Definition}[section]
\newtheorem{assumption}{Assumption}[section]
\newtheorem{remark}{Remark}[section]
\newtheorem{proposition}{Proposition}[section]
\newtheorem{corollary}{Corollary}[section]

\definecolor{down}{HTML}{B1281C}
\definecolor{up}{HTML}{489638}
\definecolor{darkblue}{HTML}{1F33B4}
\definecolor{ratioblue}{HTML}{266DA8}
\definecolor{ratiored}{HTML}{CD3334}

\newcommand{\ie}{\textit{i.e.}}

\newtcolorbox{templatebox}[1]{
    breakable,
    enhanced,
    colback=white,
    colframe=gray!80!black,
    colbacktitle=gray!80!black,
    coltitle=white,
    fonttitle=\bfseries,
    title=#1,
    arc=3mm,
    boxrule=1pt,
    drop fuzzy shadow={gray!50!white},
    left=5mm,
    right=5mm,
    top=3mm,
    bottom=3mm
}

\newtcblisting{rawtemplatebox}[2][]{
    breakable,
    enhanced,
    listing only,
    listing engine=listings,
    colback=white,
    colframe=gray!80!black,
    colbacktitle=gray!80!black,
    coltitle=white,
    fonttitle=\bfseries,
    title=#2,
    arc=3mm,
    boxrule=1pt,
    drop fuzzy shadow={gray!50!white},
    left=2mm,
    right=2mm,
    top=1.5mm,
    bottom=1.5mm,
    listing options={
        basicstyle=\ttfamily\scriptsize,
        breaklines=true,
        breakautoindent=false,
        breakindent=0pt,
        columns=fullflexible,
        keepspaces=true,
        showstringspaces=false,
        upquote=true
    },
    #1
}

\newtcolorbox{actguidecasebox}[1]{
    enhanced,
    breakable,
    colback=white,
    colframe=black,
    colbacktitle=gray!20,
    coltitle=black,
    fonttitle=\bfseries\scriptsize,
    title=#1,
    boxrule=0.5pt,
    arc=0mm,
    left=1mm,
    right=1mm,
    top=1mm,
    bottom=1mm,
    before skip=0.8\baselineskip,
    after skip=0.8\baselineskip,
    fontupper=\ttfamily\scriptsize\raggedright\sloppy,
    before upper={\setlength{\parskip}{2pt}\setlength{\parindent}{0pt}\obeylines}
}

\newcommand{\actguidecaseheader}[1]{%
    \noindent\colorbox{gray!20}{\parbox{\dimexpr\linewidth-2\fboxsep\relax}{\ttfamily\scriptsize\textbf{#1}}}\par
}

\title{Learning Agentic Policy from Action Guidance}

\author{
    Yuxiang Ji\textsuperscript{1,2}\thanks{Equal contribution. Work done during internship at AMAP, Alibaba Group. \quad \textsuperscript{\textdagger}Project lead.}\quad 
    \textbf{Zengbin Wang}\textsuperscript{2}\footnotemark[1]\quad
    \textbf{Yong Wang}\textsuperscript{2}\textsuperscript{\textdagger}\quad 
    \textbf{Shidong Yang}\textsuperscript{2}\quad
    \textbf{Ziyu Ma}\textsuperscript{2} \\
    \textbf{Guanhua Chen}\textsuperscript{3}\quad
    \textbf{Zonghua Sun}\textsuperscript{1}\quad
    \textbf{Liaoni Wu}\textsuperscript{1}\quad
    \textbf{Xiangxiang Chu}\textsuperscript{2} \\
    \textsuperscript{1}Xiamen University 
    \textsuperscript{2}AMAP, Alibaba Group
    \textsuperscript{3}Southern University of Science and Technology \\
}

\begin{document}

\maketitle
\setcounter{footnote}{0}

\begin{center}
\vspace{-30pt}
  \textbf{\faGithub\ GitHub:} \href{https://github.com/AMAP-ML/ActGuide-RL}{\textcolor{ratioblue}{https://github.com/AMAP-ML/ActGuide-RL}}
\end{center}

\input{sec/0_abstract}

\input{sec/1_introduction}

\input{sec/3_method_v2}

\input{sec/4_experiment}

\input{sec/2_related_work}

\input{sec/5_conclusion}

\bibliographystyle{plainnat}
\bibliography{ref}

\appendix
\input{sec/6_appendix}


\end{document}

%% file: sec/0_abstract.tex
\vspace{10pt}
\begin{abstract}

    Agentic reinforcement learning (RL) for Large Language Models (LLMs) critically depends on the exploration capability of the base policy, as training signals emerge only within its in-capability region.
    For tasks where the base policy cannot reach reward states, additional training or external guidance is needed to recover effective learning signals.
    Rather than relying on costly iterative supervised fine tuning (SFT), we exploit the abundant action data generated in everyday human interactions.
    We propose \textsc{ActGuide-RL}, which injects action data as plan-style reference guidance, enabling the agentic policy to overcome reachability barriers to reward states.
    Guided and unguided rollouts are then jointly optimized via mixed-policy training, internalizing the exploration gains back into the unguided policy.
    Motivated by a theoretical and empirical analysis of the benefit-risk trade-off, we adopt a minimal intervention principle that invokes guidance only as an adaptive fallback, matching task difficulty while minimizing off-policy risk.
    On search-agent benchmarks, \textsc{ActGuide-RL} substantially improves over zero RL (+10.7 pp on GAIA and +19 pp on XBench with Qwen3-4B), and performs on par with the SFT+RL pipeline without any cold start.
    This suggests a new paradigm for agentic RL that reduces the reliance on heavy SFT data by using scalable action guidance instead.
    
\end{abstract}

\begin{figure}[htbp]
    \centering
    \vspace{-20pt}
    \includegraphics[width=\linewidth]{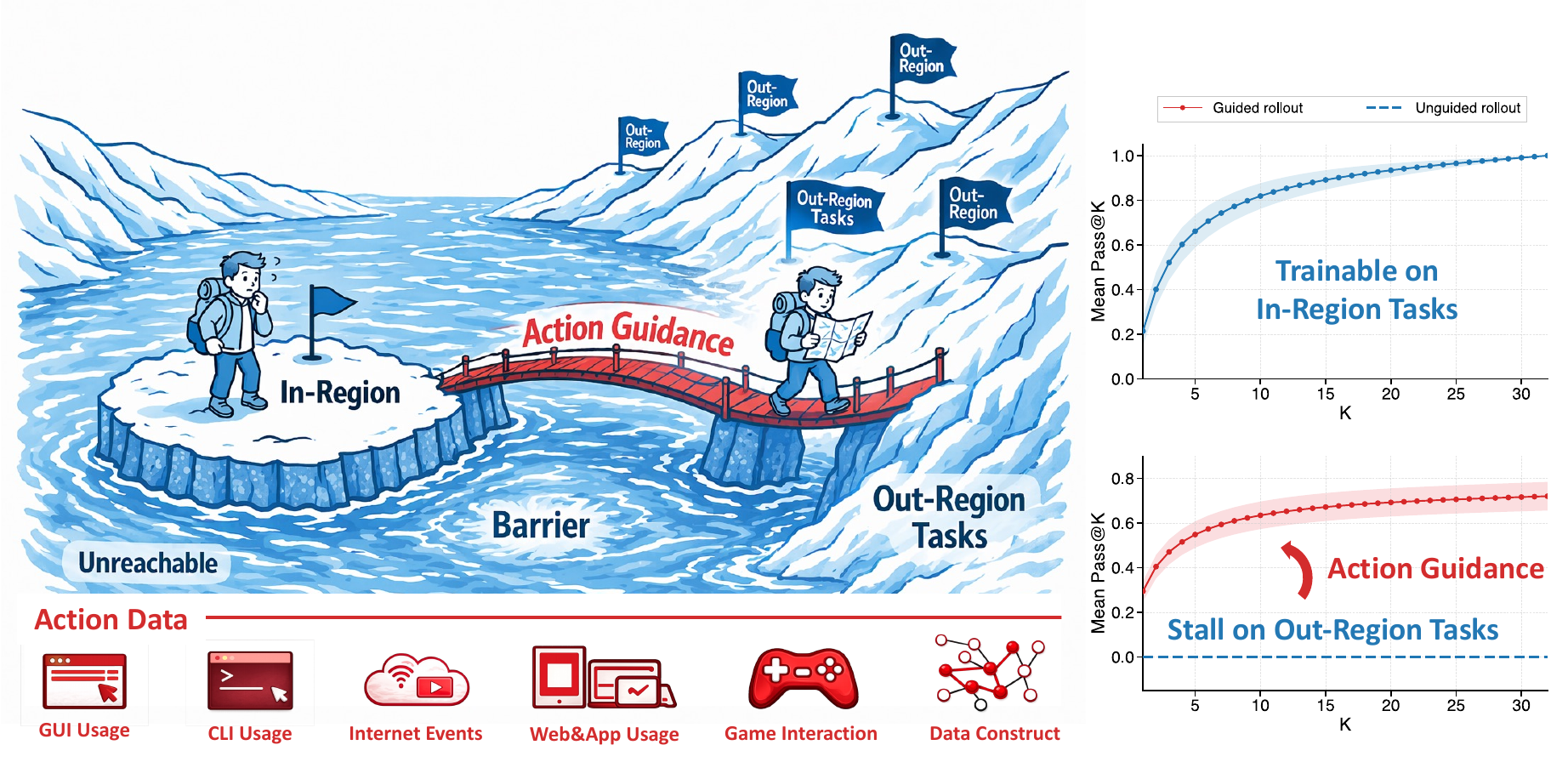}
    \caption{
        Agentic RL is typically confined to the \textit{in-capability region}\protect\footnotemark{} of the base policy, and stalls on \textit{out-region} tasks beyond this exploration frontier.
        \textsc{ActGuide-RL} leverages diverse and scalable \textbf{action data} as plan-style reference to guide effective state visitation in \textit{out-region} tasks.
    }
    \label{fig:teaser}
\end{figure}

\footnotetext{\textit{In-region} is where reward signals are reachable during rollout (pass@K > 0).}

%% file: sec/1_introduction.tex
\section{Introduction}
\vspace{-5pt}

The role of Large Language Models (LLMs) has shifted from simple chatbots to agents capable of independently solving complex tasks~\citep{yao2022react,yao2024tau,luo2025large,wei2026agentic,ma2026skillclaw}.
With targeted agentic training, recent frontier models can autonomously plan and accomplish a wide range of complex tasks~\citep{openai_gpt_5_4,anthropic_opus_4_6,team2026kimi}.
This ability has been demonstrated in general tool-use~\citep{barres2025tau,dong2025tool,ji2026thinking}, GUI~\citep{OSWorld,qin2025ui,zheng2026code2world}, and CLI~\citep{jimenez2023swe} settings, including in-the-wild real-world scenarios~\citep{wang2026openclaw,wildclawbench}.
A key factor behind such targeted training is agentic reinforcement learning (RL), in which LLM-based policies are optimized through repeated interaction with specific or diverse environments toward verifiable or heuristic rewards~\citep{zhang2025landscape,wang_ragen_2025,ji2026tree}.

Unlike static supervised training, online RL is highly sensitive to task difficulty because the training signal comes only from exploration by the model itself.
As Figure~\ref{fig:teaser}, we refer to tasks within the reachable capability of the base policy as \textit{in-region}, and those beyond this boundary as \textit{out-region}.
When reward states fall into the out-region, group-based advantage estimates collapse to zero gradient, causing training to stall.
As a result, a common view is that current RL-based methods are fundamentally limited by the capabilities of the base model~\citep{yue_does_2025,wu2026learn,dai2026harder,huang2025boosting}.
To address the cold-start problem of RL on difficult or unseen tasks, a typical practice is to perform corresponding Supervised Fine-Tuning (SFT) followed by dynamic difficulty adjustment or curriculum learning.
However, such pipelines shift the burden to warm-start data and careful curriculum design.
This dependence makes agentic RL complex and difficult to scale to new environments.

Stepping back to the original motivation for developing agentic capabilities, the goal is to move beyond reasoning and enable models to act, interact, and make decisions in a human-like manner to accomplish long-horizon tasks.
From this perspective, a direct and currently underutilized training source is the abundant \textit{action data} generated in open-world settings or during task construction.
As shown in Figure~\ref{fig:teaser}, examples include step-by-step GUI/CLI interactions with computers or phones, API-mediated task execution, and long-horizon gameplay.
In addition, some agentic RL tasks are constructed through a reverse process~\citep{li_websailor_2025,dong2026agent,jimenez2023swe}, where a valid action trajectory is first constructed and then used to instantiate the task, making the correct actions naturally available.
These action data are inherently diverse and large in scale, yet their direct use for model training is often limited by the absence of explicit reasoning traces.
Existing approaches either augment such data with synthesized chain-of-thought~\citep{erdogan2025plan,yang2026gui} or directly leverage it through behavior imitation~\citep{deng2023mind2web,caccia2024fine}. 
However, synthesized reasoning can suffer from post-hoc rationalization~\citep{turpin2023language}, while behavior imitation tends to fit surface action patterns rather than induce the reasoning abilities of agentic policy.

In this work, we investigate how to leverage action data to enhance agentic RL.
Through empirical analysis, we first characterize the capability barrier of agentic policies, where reward states fall outside the current reachable region and training signals become unavailable.
To address this issue, we propose \textsc{ActGuide-RL}, which injects action data as plan-style reference guidance to help the policy cross such barriers and perform effective out-region state visitation.
We further analyze the benefit-risk trade-off introduced by guidance, where stronger guidance improves exploration but also increases off-policy distribution shift.
Based on this, we draw two main conclusions from our experiments:
\textbf{C1}: Action guidance works best when it serves as a zero-reward fallback and is minimized adaptively, following a principle of minimal intervention.
\textbf{C2}: Under such minimal intervention, guided rollouts can be directly internalized into the unguided model through a mixed-policy optimization paradigm.

We evaluate \textsc{ActGuide-RL} on search-agent benchmarks across different base models, task difficulty levels, and both in-domain and out-of-domain settings.
Compared with zero RL, \textsc{ActGuide-RL} consistently improves all tested base models, with especially large gains on harder benchmarks where unguided RL struggles to obtain effective training signals.
Specifically, based on Qwen3-4B-Instruct, \textsc{ActGuide-RL} improves zero RL by +10.68 pp on GAIA, +27.79 pp on WebWalkerQA, +19.00 pp on XBench, and +5.15 pp on BC-ZH.
Notably, it also performs on par with the SFT+RL pipeline even without any cold-start initialization.
This substantially alleviates the dependence on SFT and offers a new perspective for agentic post-training.

%% file: sec/3_method_v2.tex
\section{Method}
\vspace{-5pt}
    

\subsection{Preliminaries: Agentic RL}
\vspace{-5pt}
    We follow existing works to formulate Agentic RL as a Partially Observable Markov Decision Process (POMDP), where a language model acts as a policy $\pi_\theta$.
    Given a task instance $x \sim \mathcal{D}$, the policy receives the interaction history as its state $s_t$ at each step $t$, and predicts the next step $\alpha_t \sim \pi_\theta(\cdot \mid s_t)$.
    A full rollout yields a trajectory $\tau$ with a binary outcome reward $Y(\tau) \in \{0, 1\}$ indicating whether the task is successfully solved.
    The overall training objective is to maximize the expected reward:
    \begin{equation}
        \max_\theta \;\; \mathcal{J}(\theta)
        \;:=\; \mathbb{E}_{x \sim \mathcal{D}}\, \mathbb{E}_{\tau \sim \pi_\theta(\cdot \mid x)}\!\left[\, Y(\tau)\, \right].
        \label{eq:objective}
    \end{equation}
    Since $Y(\tau)$ is binary, this naturally amounts to maximizing the expected success rate over a task distribution that may contain tasks of \emph{varying difficulty}.

    \begin{figure}[t]
        \centering
        \includegraphics[width=\linewidth]{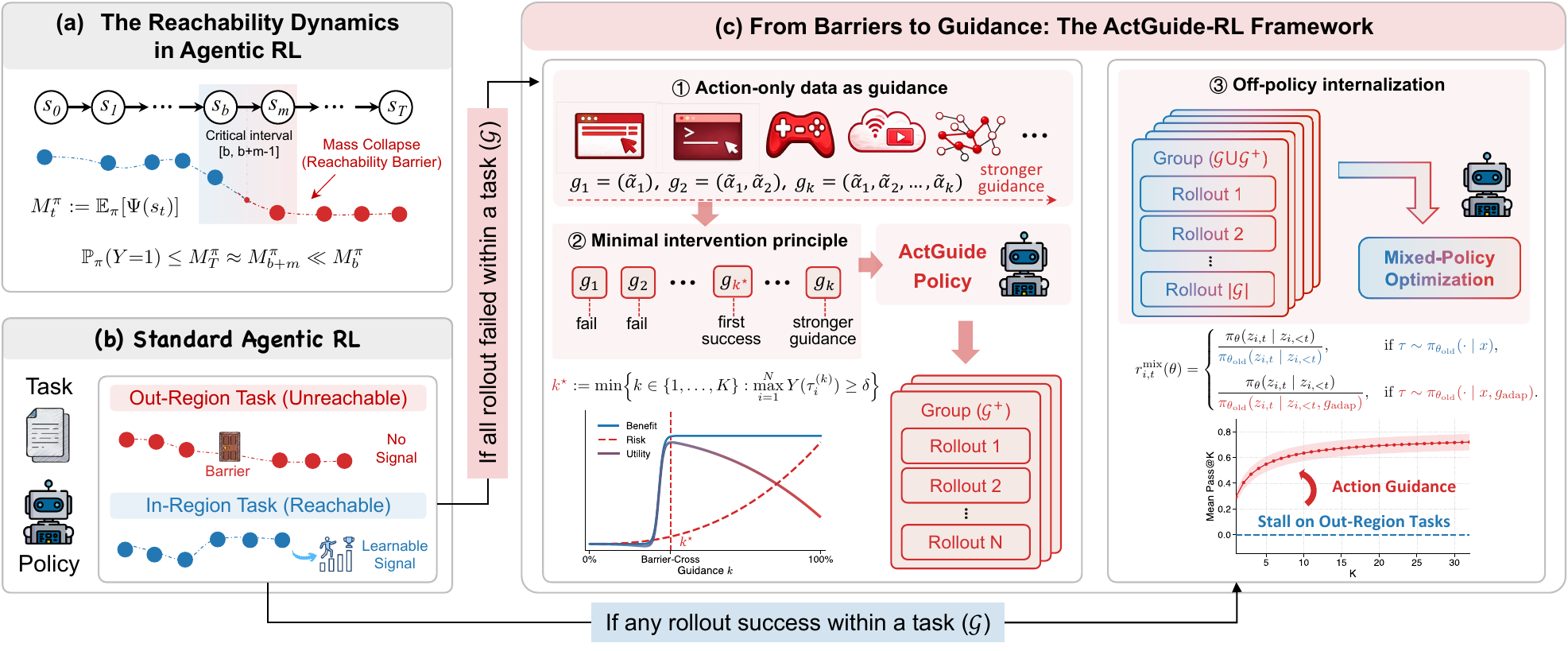} 
        \caption{
            \textbf{Overview of \textsc{ActGuide-RL} framework.}
            Conventional agentic RL can only obtain training signals within the base model in-capability region.
            \textsc{ActGuide-RL} follows the principle of minimal intervention, dynamically introducing action data to guide the model toward out-region exploration.
            Such mixed rollouts are trained through mixed-policy optimization.
        }
        \vspace{-10pt}
        \label{fig:framework}
    \end{figure}
    
\subsection{The Reachability Barrier in Agentic RL}
\label{sec:reachability_dynamics}
\vspace{-5pt}
    %
    %
    To optimize the above objective, recent RL algorithms~\citep{chu2026gpg,shao_deepseekmath_2024,yu_dapo_2025} often sample a group of $N$ rollout trajectories $\{\tau_i\}_{i=1}^N$ per task and compute advantages from the contrast between successful and failed ones.
    This mechanism works well when reward states lie within the in-capability region.
    However, when reward states fall into the out-region and become unreachable, no learning signal is obtained.
    We formalize this phenomenon through the concept of \emph{\textbf{reachability dynamics}}.

    \begin{definition}[\textbf{Reachability Dynamics}]\label{def:reachability_dynamics}
        Let $\Psi(s) := \sup_{\pi} \mathbb{P}_{\pi}(Y {=} 1 \mid s)$ denote the least upper bound on the success probability achievable by any continuation policy from state $s$.
        We define the effective state-visiting mass
        \begin{equation}
            M_t^\pi := \mathbb{E}_{\pi}[\Psi(s_t)],
        \end{equation}
        which measures the average remaining success potential along rollouts induced by policy $\pi$.
        The ratio $\bar{\kappa}_t^\pi := M_{t+1}^\pi / M_t^\pi$ quantifies the one-step reachability retention.
        By telescoping, the mass over any interval $[u, v)$ satisfies the multiplicative recursion 
        \begin{equation}
            M_v^\pi = M_u^\pi \prod_{t=u}^{v-1} \bar{\kappa}_t^\pi.
        \end{equation}
    \end{definition}

    \begin{remark}[\textbf{Mass Collapse as Reachability Barrier}]\label{rmk:barrier}
        Since $\Psi(s)$ upper-bounds the success probability achievable from $s$ under any policy, the terminal success probability satisfies
        \begin{equation}
            \mathbb{P}_\pi(Y{=}1) = \mathbb{E}_\pi\!\left[\mathbb{P}_\pi(Y{=}1 \mid s_T)\right] \leq \mathbb{E}_\pi\!\left[\Psi(s_T)\right] = M_T^\pi.
            \label{eq:mass_upperbound}
        \end{equation}
        Suppose a short critical interval $[b, b{+}m{-}1]$ exhibits low cumulative retention, \ie, $\prod_{t=b}^{b+m-1} \bar{\kappa}_t^\pi \ll 1$, so that $M_{b+m}^\pi \ll M_b^\pi$.
        Once such a sharp drop in reachability mass occurs, the remaining rollout tends to stay in low-reachability regions, so the terminal mass remains close to the post-collapse level, \ie, $M_T^\pi \approx M_{b+m}^\pi$.
        We call such a regime a \textbf{reachability barrier}.
    \end{remark}

    A reachability barrier makes rollouts beyond step $b{+}m$ receive $Y(\tau){=}0$, collapsing the group-based advantage to zero gradient. 
    \emph{This confines the model to in-region training and prevents learning on out-region tasks.}
    Unlike insufficient sampling, this failure is structural, so increasing $N$ cannot help. 
    The policy itself must first be steered across the critical interval, which motivates our method below.

    

\subsection{From Barriers to Guidance: The \textsc{ActGuide-RL} Framework}
\vspace{-5pt}
    %
    To address the fundamental barrier in agentic RL, we propose \textsc{ActGuide-RL} to use action as guidance, illustrated in Figure~\ref{fig:framework}. 
    \textsc{ActGuide-RL} is driven by three core questions along with two empirical findings: whether action data can repair reachability barriers (\S\ref{sec:barrier_repair}, Finding~1), how much guidance to inject (\S\ref{sec:minimal_intervention}, Finding~2), 
    and how to optimize from guided samples (\S\ref{sec:offpolicy}).

\subsubsection{How to Guide: Action Data Repairs Barriers} \label{sec:barrier_repair}
\vspace{-5pt}
    To explore whether action-only data can repair reachability barriers, we treat the action trajectory as a reference plan $g=(\tilde{\alpha}_1,\dots,\tilde{\alpha}_L)$ and condition the policy as $\pi_\theta(\cdot \mid s, g)$. 
    We then compare the guided and unguided behavior along the guided rollout. 
    Specifically, we measure:
        \begin{equation}
            \resizebox{0.9\linewidth}{!}{$
                \underbrace{
                |\Delta\mathrm{Logit}|
                =
                \left|
                \mathrm{logit}_{\pi_\theta}(\cdot \mid s_t, g)
                -
                \mathrm{logit}_{\pi_\theta}(\cdot \mid s_t)
                \right|
                \vphantom{\mathbb{P}_{\tau_{1:K}\sim\pi_\theta(\cdot\mid s_t)}
                \!\left[\max_{i\le K}Y(\tau_i)=1\right]}
                }_{\text{token-level guidance influence}},
                \; \;
                \underbrace{
                \mathrm{Pass@K}
                =
                \mathbb{P}_{\tau_{1:K}\sim\pi_\theta(\cdot\mid s_t)}
                \!\left[\max_{i\le K}Y(\tau_i)=1\right]
                }_{\text{prefix-level reachability}}
            $}
        \end{equation}
    where $|\Delta\mathrm{Logit}|$ computes the token logits difference between the guided policy $\pi_\theta(\cdot \mid s,g)$ and the unguided policy $\pi_\theta(\cdot \mid s)$, capturing how much guidance changes the policy locally.
    Prefix-level $\mathrm{Pass@K}$ instead samples unguided continuations from the current guided state $s_t$ and measures whether they can recover reward, reflecting the remaining reachability after that state.
    
    \begin{figure}[htbp]
        \vspace{0pt}
        \centering
        \includegraphics[width=\linewidth]{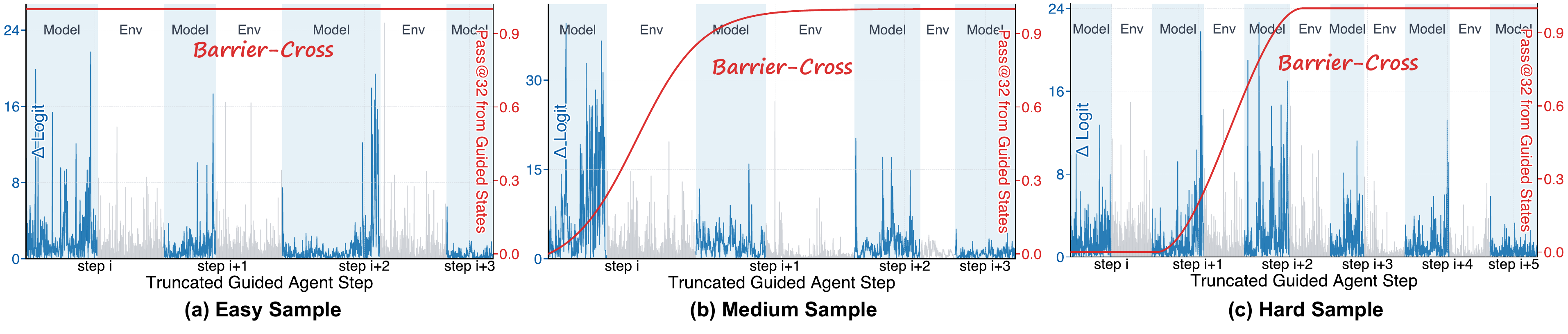}
        \vspace{-18pt}
        \caption{
            Action guidance repairs reachability barriers along guided rollouts.
            \textcolor[HTML]{236DA9}{Blue bars} show $|\Delta\mathrm{Logit}|$, and \textcolor[HTML]{D53535}{red curves} show prefix-level Pass@K ($K{=}32$).
            Barriers emerge where unguided Pass@K collapses and the guidance-induced logit shift spikes.
        }
        \label{fig:barrier_sample}
        \vspace{-5pt}
    \end{figure}

    \noindent\textbf{Finding 1: Action guidance repairs reachability barriers.}
    As shown in Figure~\ref{fig:barrier_sample}, easy tasks\footnote{Easy samples: the model can discover reward from early guided states.} already show non-zero Pass@K from early guided states, while harder tasks\footnote{Harder samples: rewarding states become reachable only at much later guided states.} keep zero unguided Pass@K until the guided trajectory crosses the barrier.
    Within these barrier intervals, $|\Delta\mathrm{Logit}|$ spikes sharply, showing that action trajectories diverge from the current policy exactly where it fails.
    After the barrier is crossed, unguided Pass@K recovers to non-trivial levels, showing that action guidance brings the policy to reachable reward states rather than simply replacing its decisions.
    %
    
    Motivated by Finding~1, we formally leverage action data ($g$) as the effective guidance signal and simply append it to the task prompt as a list of future reference actions (Appendix~\ref{box:ActGuide prompt}). 
    This provides a non-intrusive reference plan, rather than forcing the model to generate the actions as a fixed prefix.
    Moreover, recognizing that different barriers may require varying amounts of guidance to cross, we organize guidance into an ordered family
    \begin{equation}
        g_0 = \varnothing \prec g_1 \prec \cdots \prec g_K,
    \end{equation}
    where $g_k = (\tilde{\alpha}_1, \dots, \tilde{\alpha}_k)$ provides the first $k$ reference actions. This gives guidance a monotone strength parameter, which later allows us to search for the minimal sufficient intervention.
    For a barrier interval $[b, b+m-1]$ of the base policy $\pi_\theta(\alpha_t \mid s_t)$, 
    we measure the \emph{barrier-repair benefit} of guidance level $g_k$ by the increase of effective state-visiting mass after the barrier:
    \begin{equation}
        B_k := \log \frac{M_{b+m}^{\pi_\theta(\cdot \mid s, g_k)}}{M_{b+m}^{\pi_\theta(\cdot \mid s)}},
        \label{eq:barrier_repair}
    \end{equation}
    where a larger $B_k$ implies that the guidance better preserves reachable success potential.
    
    %
    %

\subsubsection{How Much to Guide: Minimal Intervention Principle} \label{sec:minimal_intervention}
\vspace{-5pt} 
    While stronger guidance raises the barrier-repair benefit $B_k$ (Eq.~\ref{eq:barrier_repair}), it also induces a larger distribution shift from the base policy, increasing the risk of off-policy optimization error~\citep{van2018deep,zheng2025prosperity}.
    Let $\tau=(z_1,\ldots,z_{|\tau|})$ be the generated token sequence.
    To quantify the distribution shift under guidance level $g_k$, we measure the cumulative token-level log-ratio shift of a rollout $\tau$:
    \begin{equation}
        \mathcal{L}_k(\tau)
        :=
        \sum_{j=1}^{|\tau|}
        \log
        \frac{\pi_\theta(z_j\mid z_{<j})}
             {\pi_\theta(z_j\mid z_{<j}, g_k)}.
    \end{equation}
    The corresponding \emph{off-policy risk} is the variance of this shift:
    \begin{equation}
        R_k := \mathrm{Var}_{\tau \sim \pi_\theta(\cdot \mid s, g_k)}\!\left(\mathcal{L}_k(\tau)\right).
    \end{equation}

    \begin{wrapfigure}{r}{0.42\linewidth}
        \vspace{-2pt}
        \centering
        \includegraphics[width=\linewidth]{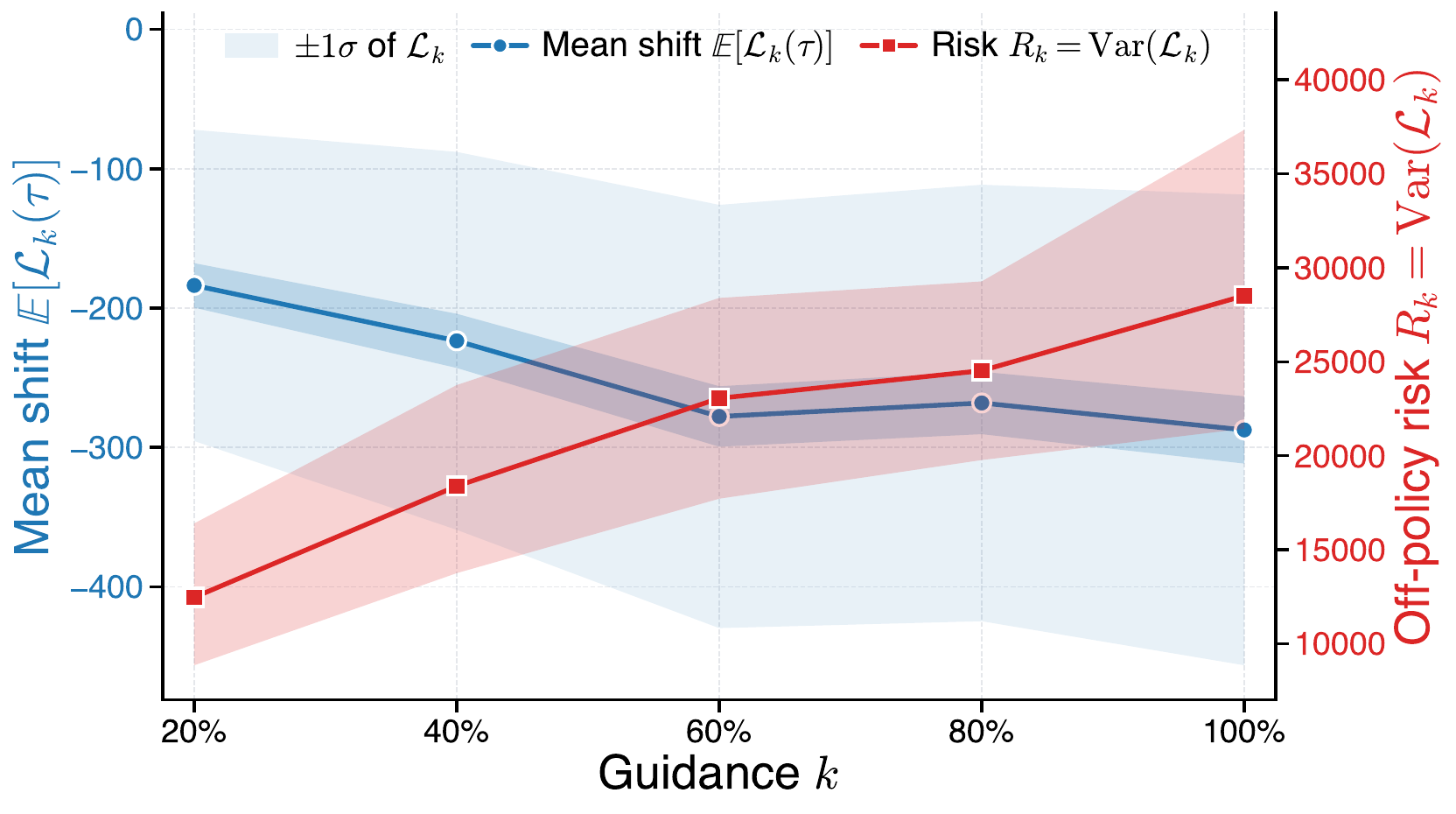}
        \caption{
            %
            Guidance-induced distribution shift under different guidance proportions. The blue curve shows the mean log-ratio shift, while the red curve shows its variance, \ie, the off-policy risk $R_k$.
        }
        \label{fig:guidance_tradeoff}
        \vspace{-10pt}
    \end{wrapfigure}

    \vspace{2pt}
    \noindent\textbf{Finding 2: Over-guidance inflates off-policy risk.}
    As shown in Figure~\ref{fig:guidance_tradeoff}, the mean log-ratio shift (blue) and its variance (red) describe the guidance-induced distribution shift from complementary perspectives.
    As the guidance level $k$ increases, the off-policy risk $R_k$ keeps rising, indicating that stronger guidance makes guided rollouts increasingly unstable for off-policy optimization.

    Motivated by Finding~2, we adopt a \emph{minimal intervention principle}: for each task, use the least guidance level that recovers reward. 
    This principle can be viewed as approximately maximizing a guidance utility $J_k = B_k - \lambda R_k$, where the barrier-repair benefit $B_k$ exhibits a sharp increase after reward recovery while the off-policy risk $R_k$ grows with the guidance level.
    In practice, we first collect an unguided rollout group per task, invoking guidance only as a fallback when the entire group fails. 
    Under a mild monotonicity assumption (stronger levels do not decrease recovery probability), we can efficiently identify the smallest sufficient level $k^\star$ via binary search:
    \begin{equation}
        k^\star := \min \Bigl\{ k \in \{1, \dots, K\} : \max_{i=1}^N Y(\tau_i^{(k)}) \ge \delta \Bigr\},
        \label{eq:binary_search}
    \end{equation}
     where $\{\tau_i^{(k)}\}_{i=1}^N$ are $N$ rollouts under guidance $g_k$ and $\delta > 0$ is the success threshold. 
    We denote the resulting adaptive guidance as $g_\text{adap} := g_{k^\star}$, which keeps guided rollouts close to the unguided distribution and enables the off-policy optimization studied next.
    Note that under binary rewards, $B_k$ exhibits threshold behavior (near zero until the barrier is crossed, then jumping sharply), while $R_k$ grows monotonically. The guidance utility $J_k$ therefore peaks near the minimal successful level, making the binary search in Eq.~\ref{eq:binary_search} a practical proxy for approximately maximizing $J_k$.


\subsubsection{How to Learn: Off-Policy Internalization} \label{sec:offpolicy}
\vspace{-5pt}
    Action guidance is available only at training time. At inference, the agent must act under the unguided policy $\pi_\theta(\cdot \mid x)$, so any learning signal extracted from guided rollouts has to be internalized.
    Since the guided policy $\pi_\theta(\cdot \mid x, g)$ shares parameters with the unguided one, we treat guided samples as off-policy data w.r.t.\ $\pi_\theta(\cdot \mid x)$ and optimize the mixed objective
    \begin{equation}
        \small
        \begin{aligned}
        \mathcal{J}_{\mathrm{mix}}(\theta)
        &=
        \mathbb{E}_{(x,\bar g)\sim \mathcal D,\ \mathcal G\sim q_{\theta_{\rm old}}^{\rm mix}(\cdot \mid x,\bar g)}
        \Biggl[
        \frac{1}{\sum_i T_i}
        \sum_{i=1}^{|\mathcal{G}|}\sum_{t=1}^{T_i}
        \min\Bigl(
        r_{i,t}^{\rm mix}(\theta)\,\hat{A}(\tau_i), \\
        &\qquad\qquad
        \mathrm{clip}\!\left(r_{i,t}^{\rm mix}(\theta), 1{-}\epsilon, 1{+}\epsilon\right)\hat{A}(\tau_i)
        \Bigr)
        - \beta \frac{1}{|\mathcal G|}\sum_{i=1}^{|\mathcal G|}
        \mathbb{D}_{\mathrm{KL}}\!\left(
        \pi_\theta(\tau_i \mid x)\,\|\,\pi_{\mathrm{ref}}(\tau_i \mid x)
        \right)
        \Biggr],
        \end{aligned}
    \end{equation}
    where $q_{\theta_{\rm old}}^{\rm mix}$ denotes the mixed rollout collection process in Algorithm~\ref{alg:adaptive_guidance}, $\hat{A}(\tau_i)$ is the group-based advantage, and the token-level importance ratio adapts to the rollout source:
    \begin{equation}
        r^{\rm mix}_{i,t}(\theta)
        =
        \begin{cases}
        \displaystyle
        \frac{\pi_\theta(z_{i,t}\mid z_{i,<t})}
             {\textcolor{ratioblue}{\pi_{\theta_{\rm old}}(z_{i,t}\mid z_{i,<t})}},
        & \text{if } \textcolor{ratioblue}{\tau_i \sim \pi_{\theta_{\rm old}}(\cdot \mid x)}, \\[1.2em]
        \displaystyle
        \frac{\pi_\theta(z_{i,t}\mid z_{i,<t})}
             {\textcolor{ratiored}{\pi_{\theta_{\rm old}}(z_{i,t} \mid z_{i,<t}, g_{\text{adap}})}},
        & \text{if } \textcolor{ratiored}{\tau_i \sim \pi_{\theta_{\rm old}}(\cdot \mid x,g_{\text{adap}})}.
        \end{cases}
    \end{equation}
    For unguided rollouts this is the standard importance ratio; for guided rollouts the denominator uses the guided distribution, transferring credit back to the unguided target $\pi_\theta(\cdot \mid x)$.
    Unlike prior off-policy RL methods that include ratio shaping~\citep{yan_learning_2025,nath2025adaptive}, we keep the optimization objective unchanged because minimal intervention limits the shift between guided rollouts and the base policy.
    
    %

    \begin{algorithm}[t]
    \caption{Adaptive Minimal-Intervention Training with Action Guidance}
    \label{alg:adaptive_guidance}
    \KwIn{policy $\pi_\theta$, dataset $\mathcal D=\{(x,\bar g)\}$, minibatch size $M$, group size $N$, threshold $\delta$, search budget $B$, steps $S$}
    \For{$s=1$ \KwTo $S$}{
        Sample $\mathcal B=\{(x_b,\bar g_b)\}_{b=1}^M \sim \mathcal D$; $\mathcal G \leftarrow \emptyset$\;
        \ForEach{$(x_b,\bar g_b)\in\mathcal B$}{
            Define $g_{b,k}=(\tilde\alpha_{b,1},\ldots,\tilde\alpha_{b,k})$\;
            $\mathcal G_b \leftarrow \{(\tau_{b,i},r_{b,i})\}_{i=1}^N,\ \tau_{b,i}\sim\pi_{\theta_{\rm old}}(\cdot\mid x_b),\ r_{b,i}=Y(\tau_{b,i})$\;
            \If{$\max_i r_{b,i} < \delta$}{
                $k_b^\star \leftarrow \min\Bigl\{k:\max_j r_{b,j}^{(k)} \ge \delta\Bigr\}$ via binary search under budget $B$\;
                \If{$k_b^\star$ exists}{
                    $\mathcal G_b^+ \leftarrow \{(\tau_{b,i}^+,r_{b,i}^+)\}_{i=1}^N,\ \tau_{b,i}^+\sim\pi_{\theta_{\rm old}}(\cdot\mid x_b,g_{b,k_b^\star}),\ r_{b,i}^+=Y(\tau_{b,i}^+)$\;
                    $\mathcal G_b \leftarrow \mathcal G_b \cup \mathcal G_b^+$\;
                }
            }
            $\mathcal G \leftarrow \mathcal G \cup \mathcal G_b$\;
        }
        Compute advantages on $\mathcal G$;
        Update $\pi_\theta(\cdot\mid x)$ by $\mathcal J_{\mathrm{mix}}$\;
    }
    \end{algorithm}

%% file: sec/4_experiment.tex
\section{Experiment}
\vspace{-5pt}

\subsection{Experimental Setup}
\vspace{-5pt}

    \noindent\textbf{Benchmarks.} 
    To evaluate the effectiveness of our proposed \textsc{ActGuide-RL} in LLM agentic RL, we conduct experiments in the search-agent setting, which is stateless and facilitates the collection of action data.
    Our evaluation covers two categories of benchmarks. 
    The first category is in-domain search-agent benchmarks, including four representative datasets, \textit{GAIA}~\citep{mialon_gaia_2023}, \textit{WebWalkerQA}~\citep{wu_webwalker_2025}, \textit{XBench}~\citep{chen_xbench_2025}, and \textit{BrowseComp-ZH (BC-ZH)}~\citep{zhou2025browsecompzh}, which span diverse difficulty levels, multiple languages, and real-world multi-step reasoning scenarios.
    The second category is out-of-domain benchmarks, including \textit{GPQA}~\citep{rein2023gpqa}, \textit{TruthfulQA}~\citep{lin2022truthfulqa}, and \textit{IFEval}~\citep{zhou2023instruction}, which are used to evaluate the out-of-domain generalization ability of models beyond the search-agent setting.
    The detailed RL and SFT training data source are provided in Appendix~\ref{sec:datasets}.
    
    \noindent\textbf{Baselines.}
    Under the same evaluation protocol, we compare \textsc{ActGuide-RL} against several baselines, including foundation models~\citep{minimax_m2_1,liu2024deepseek,singh2025openai}, specified search-agent-trained models~\citep{li_websailor_2025,dong_agentic_2025,li_webthinker_2025}, and vanilla RL trained from the same backbones without action guidance.
    For the RL baseline, we adopt the standard GRPO objective with token-level policy optimization, using the same training data but without action guidance.

    \noindent\textbf{Implementation Details.} 
    Following Tongyi-DeepResearch~\citep{team2025tongyi}, we equip the agent with two tools, \textit{web-search} and \textit{web-visit}, whose schemas are included in the system prompt. 
    Given the limited interaction budget and context length in our setup, we use raw tool outputs directly without a separate summary model. 
    For both training reward and test-time evaluation, we adopt the few-shot, reference-based binary LLM-judge template from Tongyi-DeepResearch.    
    Full implementation details are provided in Appendix~\ref{sec:experiment details}.

    \begin{table}[t]
        \centering
        \caption{
            Main results of \textsc{ActGuide-RL} on search-agent benchmarks comparing foundation models, search-agent trained models, and RL baseline.
            The best results are indicated in \textbf{bold}.
        }
        \label{tab:main-comparison}
        \resizebox{\textwidth}{!}{
        \begin{tabular}{lcccccccccc}
            \toprule
            \multirow{2}{*}[-1em]{\textbf{Method}} & \multicolumn{4}{c}{\textbf{General AI Assistant}} & \multicolumn{4}{c}{\textbf{WebWalkerQA}} & \textbf{XBench} & \textbf{BC-ZH} \\
            \cmidrule(lr){2-5} \cmidrule(lr){6-9} \cmidrule(lr){10-10} \cmidrule(lr){11-11}
            & Lv.1 & Lv.2 & Lv.3 & Avg. & Easy & Med. & Hard & Avg. & Avg. & Avg. \\
            \midrule
            \multicolumn{11}{c}{\textbf{\textit{Foundation Model}}} \\
            \midrule
            MiniMax-M2.1 & - & - & - & 64.3 & - & - & - & - & 68.0 & 66.6 \\
            DeepSeek-V3.2 & - & - & - & 75.1 & - & - & - & - & 78.0 & 65.0 \\
            GPT-5 High & - & - & - & 76.4 & - & - & - & - & 77.0 & 65.0 \\
            \midrule
            \multicolumn{11}{c}{\textbf{\textit{Search-Agent-Trained Models}}} \\
            \midrule
            WebSailor-7B & - & - & - & 37.9 & - & - & - & - & 34.0 & 14.2 \\
            ARPO-8B & 53.9 & 32.7 & 16.7 & 38.8 & 26.7 & 33.3 & 29.6 & 30.5 & 25.0 & - \\
            WebThinker-32B-RL & 56.4 & 50.0 & 16.7 & 48.5 & 58.8 & 44.6 & 40.4 & 46.5 & 24.0 & 7.3 \\
            \midrule
            \multicolumn{11}{c}{\textbf{\textit{Baseline and} \textsc{ActGuide-RL}}} \\
            \midrule
            Qwen2.5-3B-Instruct & 15.38 & 7.69 & 0.00 & 9.71 & 5.00 & 7.14 & 4.58 & 5.73 & 8.00 & 2.08 \\
            \; $+$ RL & 15.38 & 7.69 & 16.66 & 11.65 & 15.00 & 15.00 & 15.83 & 15.29 & 10.00 & 2.42 \\
            \rowcolor[HTML]{F5F5F5}
            \; $+$ \textbf{\textsc{ActGuide-RL}} & \textbf{28.21} & \textbf{11.54} & \textbf{16.66} & \textbf{18.45} & \textbf{18.75} & \textbf{16.07} & \textbf{22.08} & \textbf{18.82} & \textbf{16.00} & \textbf{4.50} \\
            \; $\Delta$\footnotesize{Delta} & \footnotesize{+12.83} & \footnotesize{+3.85} & \footnotesize{+0.00} & \footnotesize{+6.80} & \footnotesize{+3.75} & \footnotesize{+1.07} & \footnotesize{+6.25} & \footnotesize{+3.53} & \footnotesize{+6.00} & \footnotesize{+2.08} \\
            \midrule
            Qwen2.5-7B-Instruct & 35.89 & 15.38 & 8.33 & 22.32 & 18.75 & 19.28 & 16.25 & 18.09 & 19.00 & 4.50 \\
            \; $+$ RL & 20.51 & 7.69 & 0.00 & 11.65 & 14.37 & 20.35 & 19.58 & 18.67 & 22.00 & 4.84 \\
            \rowcolor[HTML]{F5F5F5}
            \; $+$ \textbf{\textsc{ActGuide-RL}} & \textbf{41.02} & \textbf{17.30} & \textbf{8.33} & \textbf{25.24} & \textbf{24.37} & \textbf{21.07} & \textbf{21.66} & \textbf{22.05} & \textbf{24.00} & \textbf{8.31} \\
            \; $\Delta$\footnotesize{Delta} & \footnotesize{+20.51} & \footnotesize{+9.61} & \footnotesize{+8.33} & \footnotesize{+13.59} & \footnotesize{+10.00} & \footnotesize{+0.72} & \footnotesize{+2.08} & \footnotesize{+3.38} & \footnotesize{+2.00} & \footnotesize{+3.47} \\
            \midrule
            Qwen3-4B-Instruct & 17.94 & 17.30 & 0.00 & 15.53 & 8.75 & 3.57 & 0.83 & 3.82 & 14.00 & 7.96 \\
            \; $+$ RL & 33.33 & 25.00 & 0.00 & 25.24 & 13.12 & 13.92 & 9.17 & 12.06 & 18.00 & 15.26 \\
            \rowcolor[HTML]{F5F5F5}
            \; $+$ \textbf{\textsc{ActGuide-RL}} & \textbf{46.15} & \textbf{32.69} & \textbf{16.66} & \textbf{35.92} & \textbf{43.75} & \textbf{41.78} & \textbf{35.00} & \textbf{39.85} & \textbf{37.00} & \textbf{20.41} \\
            \; $\Delta$\footnotesize{Delta} & \footnotesize{+12.82} & \footnotesize{+7.69} & \footnotesize{+16.66} & \footnotesize{+10.68} & \footnotesize{+22.50} & \footnotesize{+30.63} & \footnotesize{+25.83} & \footnotesize{+27.79} & \footnotesize{+19.00} & \footnotesize{+5.15} \\
            \midrule
            Qwen3-8B & 43.58 & 26.92 & 16.66 & 32.03 & 41.87 & 31.78 & 26.25 & 32.20 & 32.00 & 23.52 \\
            \; $+$ RL & 46.15 & 32.69 & 25.00 & 36.89 & 43.75 & 44.64 & 39.16 & 42.50 & 33.00 & 21.79 \\
            \rowcolor[HTML]{F5F5F5}
            \; $+$ \textbf{\textsc{ActGuide-RL}} & \textbf{51.28} & \textbf{36.53} & \textbf{33.33} & \textbf{41.74} & \textbf{50.00} & \textbf{46.79} & \textbf{44.58} & \textbf{46.77} & \textbf{44.00} & \textbf{26.64} \\
            \; $\Delta$\footnotesize{Delta} & \footnotesize{+5.13} & \footnotesize{+3.84} & \footnotesize{+8.33} & \footnotesize{+4.85} & \footnotesize{+6.25} & \footnotesize{+2.15} & \footnotesize{+5.42} & \footnotesize{+4.27} & \footnotesize{+11.00} & \footnotesize{+4.85} \\
            \bottomrule
        \end{tabular}
        }
        \vspace{-15pt}
    \end{table}

\subsection{Main Results}
\vspace{-5pt}

    \noindent\textbf{Overall Comparison.}
    Table~\ref{tab:main-comparison} reports overall accuracy on four in-domain benchmarks, from which three observations stand out.
    \begin{itemize}[leftmargin=16pt,itemsep=2pt,topsep=0pt,parsep=0pt]
    \item \textbf{\textsc{ActGuide-RL} mitigates in-region RL capability regression.} 
    When the exploration difficulty of the RL training data does not match the base model, vanilla RL restricted to in-region exploration can lead to partial performance regression on some benchmarks. 
    For example, RL degrades Qwen2.5-7B-Instruct on GAIA and Qwen3-8B on BC-ZH, whereas \textsc{ActGuide-RL} alleviates these regressions through adaptive guidance and more effective state visitation.

    \item \textbf{\textsc{ActGuide-RL} improves exploration beyond the current reachable region.} 
    When vanilla RL fails to access sufficiently effective states on harder tasks, action guidance helps the policy move beyond its current reachable region and enables more effective state visitation. 
    This is most evident on Qwen3-4B-Instruct, where \textsc{ActGuide-RL} brings broad gains across all four benchmarks, with especially large improvements on WebWalker ($12.06\% \rightarrow 39.85\%$) and XBench ($18.00\% \rightarrow 37.00\%$).

    \item \textbf{\textsc{ActGuide-RL} delivers stable gains across base models.} 
    For base models with different capability levels, action guidance can adaptively help the policy access more effective states on each training sample according to its difficulty. 
    As a result, compared with vanilla RL, \textsc{ActGuide-RL} consistently improves all four base models, underscoring the strong adaptability of action guidance across different capability levels.
    \end{itemize}

    \noindent\textbf{Comparison with SFT + RL.}
    Another commonly used strategy to address training stalls caused by limited policy exploration is a targeted SFT cold start.
    To further analyze the role of \textsc{ActGuide-RL} relative to the SFT + RL paradigm, we also initialize the policy with an SFT cold start constructed by partially distilling Tongyi-DeepResearch-30B-A3b.
    This setting aims to explore a new possibility beyond the standard SFT + RL pipeline through action-level guidance, rather than merely pursuing performance improvements over a comprehensive SFT baseline.
    As shown in Table~\ref{tab:zerorl-sft-comparison}, even without any cold start, \textsc{ActGuide-RL} achieves performance comparable to the two-stage SFT+RL pipeline.
    Moreover, when built on the same cold-start initialized model, \textsc{ActGuide-RL} still can obtain additional gains from action guidance.
    Meanwhile, due to the mode-covering nature of SFT, cold-start initialization often degrades out-of-domain performance as the consistent performance drop on GPQA-CoT (Zero Shot), TruthfulQA and IFEVAL, whereas such degradation does not occur in \textsc{ActGuide-RL} with zero RL setting.
    
    Overall, \textsc{ActGuide-RL} offers a new alternative paradigm for agentic RL, alleviating the dependence on heavy SFT data throught the use of lighter-weight action data instead.

    \begin{table}[t]
        \centering
        \setlength{\tabcolsep}{4mm} 
        \caption{
            Comparison of \textsc{ActGuide-RL} and SFT + RL on in-domain and out-of-domain benchmarks.
        }
        \label{tab:zerorl-sft-comparison}
        \resizebox{\textwidth}{!}{
        \begin{tabular}{lccccccc}
            \toprule
            \multirow{2}{*}[-1em]{\textbf{Method}} & \multicolumn{4}{c}{\textbf{In-Domain}} & \multicolumn{3}{c}{\textbf{Out-of-Domain}} \\
            \cmidrule(lr){2-5} \cmidrule(lr){6-8}
            & GAIA & WebWalker & XBench & BC-ZH & GPQA-CoT (ZS) & TruthQA & IFEval \\
            \midrule
            ZeroRL & 25.24 & 12.06 & 18.00 & 15.26 & 35.45 & 62.17 & 81.33 \\
            \rowcolor[HTML]{F5F5F5}
            \; $+$ \textsc{\textbf{ActGuide}} & \textbf{35.92} & \textbf{39.85} & \textbf{37.00} & \textbf{20.41} & \textbf{36.93} & \textbf{62.30} & \textbf{82.99} \\
            \midrule
            SFT & 34.95 & 31.18 & 25.00 & 25.61 & 29.15 & 56.95 & \textbf{77.82} \\
            \; $+$ RL & 36.89 & 32.20 & 17.00 & 26.30 & \textbf{29.85} & 57.02 & 76.34 \\
            \rowcolor[HTML]{F5F5F5}
            \; $+$ \textsc{\textbf{ActGuide}} & \textbf{40.77} & \textbf{37.06} & \textbf{25.00} & \textbf{28.02} & 29.57 & \textbf{57.11} & 77.43 \\
            \bottomrule
        \end{tabular}
        }
        \vspace{-5pt}
    \end{table}
    
\subsection{Further Analysis and Ablation}
\vspace{-5pt}

\noindent\textbf{Training Dynamics.}
To further analyze the eff-
\vspace{-5pt}

\begin{wrapfigure}{r}{0.5\linewidth}
    \vspace{-25pt}
    \centering
    \includegraphics[width=\linewidth]{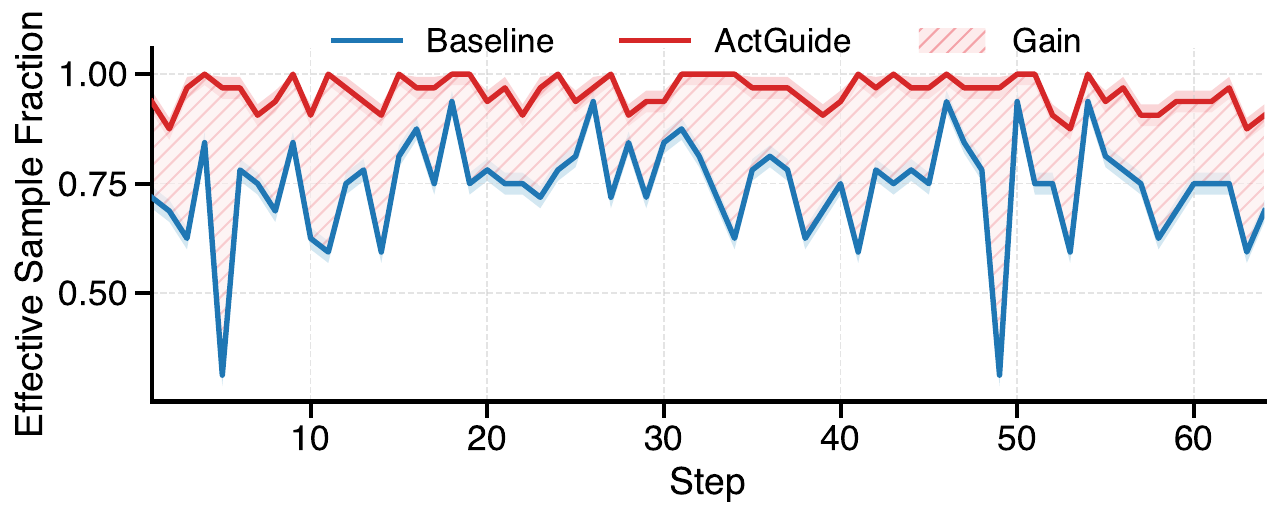}
    \vspace{-15pt}
    \caption{
        Trainable groups dynamic.
    }
    \label{fig:nonzero_dynamics}
    \vspace{-10pt}
\end{wrapfigure}

ect of action guidance on training dynamics, we track the proportion of rollout groups that provide effective learning signals during training, as shown in Figure~\ref{fig:nonzero_dynamics}. 
Specifically, we find action data helps the policy discover effective training signals in a higher proportion of samples, while the unguided baseline is frequently hindered by exploration barriers and therefore wastes many rollouts on ineffective state visitation.
This suggests that \textsc{ActGuide-RL} improves exploration beyond the current reachable region, allowing the policy to learn from out-region tasks.

\begin{figure}[t!]
    \centering
    \begin{minipage}[c]{0.44\textwidth}
        \centering
        \includegraphics[
            width=\linewidth,
            keepaspectratio
        ]{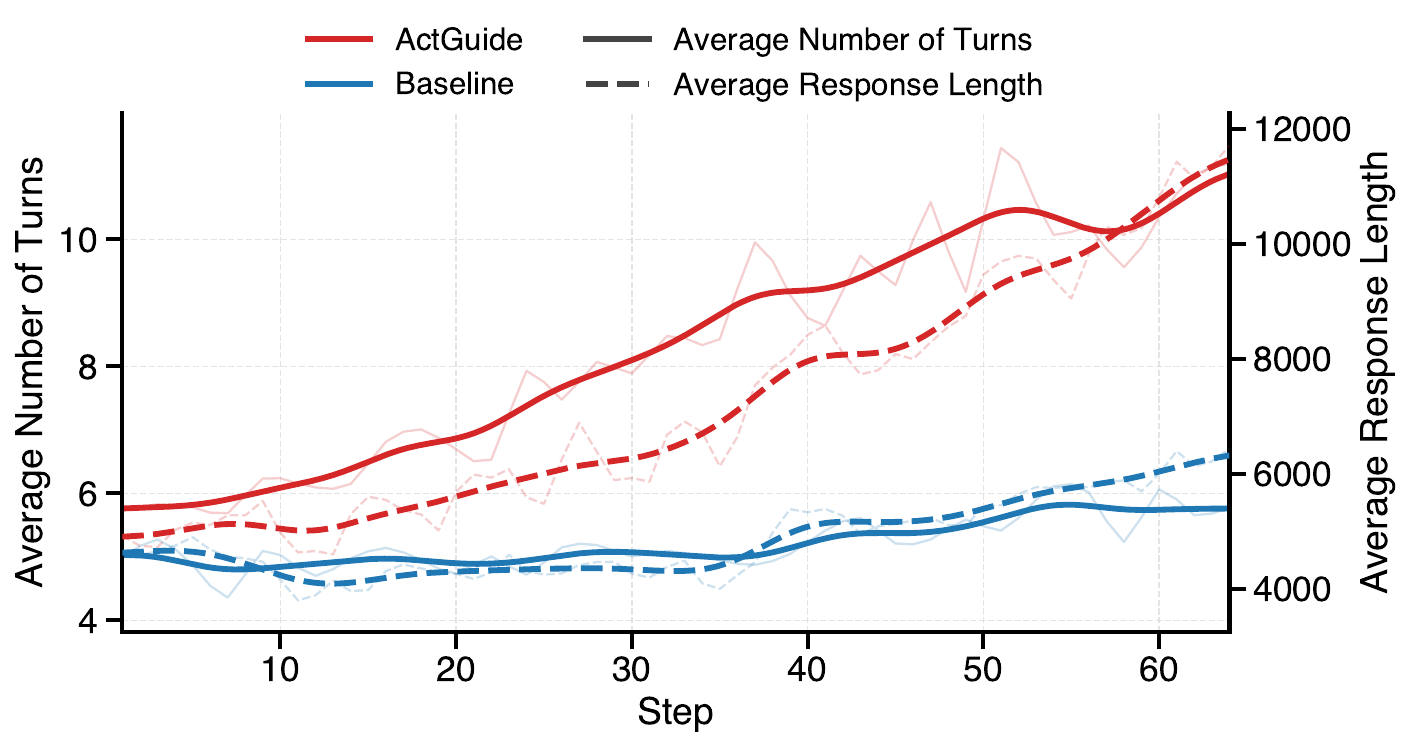}
        \vspace{-18pt}
        \captionof{figure}{
            Training dynamics on number of interaction turns and response length.
        }
        \label{fig:num turns}
    \end{minipage}
    \hfill
    \begin{minipage}[c]{0.54\textwidth}
        \centering
        \captionof{table}{
            Agent performance under different interaction turn budgets.
        }
        \label{tab:turn-budget}
        \resizebox{\linewidth}{!}{
        \begin{tabular}{lcccc}
            \toprule
            \textbf{Turn Budget} & \textbf{GAIA} & \textbf{WebWalker} & \textbf{XBench} & \textbf{BC-ZH} \\
            \midrule
            2 & 0.97 & 9.26 & 5.00 & 1.04 \\
            4 & 18.44 & 33.97 & 33.00 & 4.84 \\
            8 & 19.41 & 35.00 & 33.00 & 16.96 \\
            16 & 27.18 & 37.55 & 35.00 & 17.99 \\
            32 & 35.92 & 39.85 & 37.00 & 20.41 \\
            \bottomrule
        \end{tabular}
        }
    \end{minipage}
\end{figure}

\noindent\textbf{Towards Complex Interaction.}
A central challenge of agentic RL without cold-start is that the policy struggles to acquire complex interaction skills within its in-region tasks.
Fortunately, we find that \textsc{ActGuide-RL} enables even a small model such as Qwen3-4B-Instruct without any cold-start initialization, to gradually acquire complex interaction capability, as reflected by the steady increase in the number of interaction turns and generated tokens over training in Figure~\ref{fig:num turns}.
To further verify whether these increased interactions are indeed effective, we vary the interaction budget at evaluation time and observe that performance consistently improves as the budget increases in Table~\ref{tab:turn-budget}.

\noindent\textbf{Ablation Study on \textsc{ActGuide-RL}.} We 
\vspace{-5pt}

\begin{wrapfigure}{r}{0.55\linewidth}
    \vspace{-25pt}
    \centering
    \captionof{table}{
        Ablation study of \textsc{ActGuide-RL}. 
    }
    \label{tab:ablation}
    \vspace{0pt}
    \resizebox{\linewidth}{!}{
    \begin{tabular}{lccc}
        \toprule
        \textbf{Method} & \textbf{GAIA} & \textbf{WebWalker} & \textbf{XBench} \\
        \midrule
        \textbf{\textsc{ActGuide-RL}} & 35.92 & 39.85 & 37.00 \\
        \; $-$ Minimal-Intervention (Adaptive) & 27.18 & 35.00 & 34.00 \\
        \; $-$ Minimal-Intervention (Fallback) & 24.27 & 23.82 & 19.00 \\
        \; $-$ Mixed-Policy Optimization & 22.32 & 21.76 & 21.00 \\
        \bottomrule
    \end{tabular}
    \vspace{20pt}
    }
    \vspace{5pt}
    \includegraphics[width=\linewidth]{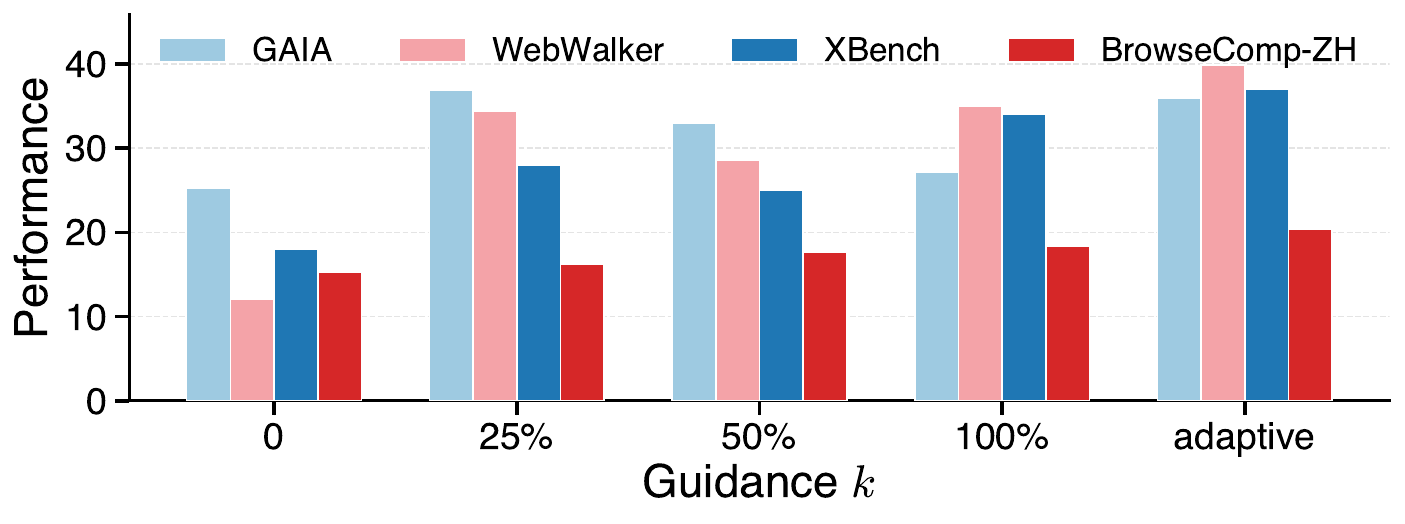}
    \vspace{-20pt}
    \captionof{figure}{
        Performance of different guidance strength.
    }
    \label{fig:guidance_k}
    \vspace{-10pt}
\end{wrapfigure}

conduct ablation studies on several key design choices in \textsc{ActGuide-RL}, including the adaptive guidance mechanism, the fallback guidance, and mixed-policy optimization.
As shown in Table~\ref{tab:ablation}, removing either the adaptive or fallback guidance mechanism causes performance degradation to different extents. 
We further compare fixed guidance ratios in Figure~\ref{fig:guidance_k}, and again find that dynamic guidance performs best.
These results indicate that action guidance is not effective simply because more guidance is provided, nor is less always better.
Rather, the best performance comes from minimally introducing guidance in an adaptive manner according to the policy capability.
Removing mixed-policy optimization also causes a substantial performance drop, since it breaks the pathway that transfers behaviors acquired under guidance into the test-time unguided capability.

\noindent\textbf{Sensitivity to Action Noise.} When consi-
\vspace{-5pt}

\begin{wraptable}{r}{0.55\linewidth}
    \vspace{-20pt}
    \centering
    \caption{
        Results of different action noise ratio.
    }
    \label{tab:noise-ratio}
    \vspace{-5pt}
    \resizebox{\linewidth}{!}{
    \begin{tabular}{lcccc}
        \toprule
        \textbf{Noise Ratio} & \textbf{GAIA} & \textbf{WebWalker} & \textbf{XBench} & \textbf{BC-ZH} \\
        \midrule
        0\% & 35.92 & \textbf{39.85} & 37.00 & \textbf{20.41} \\
        10\% & \textbf{39.81} & 39.26 & \textbf{38.00} & 19.03 \\
        20\% & 29.12 & 37.94 & 35.00 & 17.64 \\
        \bottomrule
    \end{tabular}
    }
    \vspace{-10pt}
\end{wraptable}

considering scaling up the collection of action data, an important factor is data noise, as human demonstrations may contain a substantial amount of meaningless or irrelevant actions while completing certain tasks.
Here we simulate such noise by randomly inserting task-irrelevant actions into the original per-sample action trajectories, and then perform the same \textsc{ActGuide-RL} training.
As shown in Table~\ref{tab:noise-ratio}, \textsc{ActGuide-RL} is not overly sensitive to action noise.
It maintains stable performance under a 10\% noise ratio, while a further increase to 20\% leads to a performance drop.

%% file: sec/2_related_work.tex
\section{Related Work}
\vspace{-5pt}

\subsection{Agentic RL}
\vspace{-5pt}
Recent advancements in RL~\citep{shao_deepseekmath_2024,yu_dapo_2025,schulman2017proximal,chu2026gpg,feng_group--group_2025} enable end-to-end training of agents that can interact with environments, make sequential decisions, and optimize toward long-horizon objectives.
This makes agentic RL a pivotal paradigm for both foundation-model capability building~\citep{openai_gpt_5_4,anthropic_opus_4_6,team2026kimi,qwen35blog} and domain-specific agent post-training~\citep{jin_search-r1_2025,ji2026thinking,yu2025medresearcher,team2025tongyi,chu2026redsearcher}.
Since effective agentic RL strongly depends on the base model to explore valid training signals, existing methods often rely on a cold-start before RL or on alternating SFT and RL to dynamically align the model capabilities with the target tasks~\citep{dong2025tool,pang2024iterative,chen2025sft,chen2025beyond,deng2025openvlthinker}.
Some works instead adopt dynamic task scheduling~\cite{yao2026coba,gu2026actor,zhai2025agentevolver} or curriculum learning~\cite{li2026adacurl,jiang2025vcrl} to ensure that the difficulty of training tasks is well matched to the evolving capabilities of the model.
A line of work most closely related to ours constructs curriculum learning examples from existing SFT data~\citep{yi2026pivotrl,wang2025let}, or directly uses this data as hints to guide the model toward obtaining meaningful learning signals on hard tasks~\citep{huang2025boosting,nath2025adaptive,wu2026learn}.
Unlike these approaches, \textsc{ActGuide-RL} seeks to leverage more readily available action data, offering greater practical value and stronger scaling potential.

\vspace{-5pt}
\subsection{RL from Demonstration}
\vspace{-5pt}
Our work is also related to reinforcement learning from demonstrations (RLfD)~\citep{nair2018overcoming,libardi2021guided}, where demonstrations usually take the form of expert trajectories, typically as full reasoning-and-action traces in agent settings.
Classical RLfD methods often use demonstration trajectories to bootstrap exploration in sparse-reward settings, for example by retaining them in the replay buffer and combining RL updates with auxiliary imitation losses~\citep{rajeswaran2017learning,hester2018deep,vecerik2017leveraging}.
Following a similar intuition, several recent LLM studies incorporate off-policy expert trajectories into online RL to mitigate sparse-reward and hard-exploration challenges~\citep{fu2025srft,liang2025squeeze,zhang2026onpolicyrlmeetsoffpolicy,ma2025learning}.
Specifically, LUFFY~\citep{yan2025learning} incorporates off-policy expert trajectories into online RL through mixed-policy optimization, using regularized importance shaping to avoid rigid imitation.
Guide~\citep{nath2025adaptive} utilizes adaptive hint-guided off-policy trajectories into online RL, reweighting them to improve exploration while training a policy that no longer relies on hints at inference time.
Unlike these demonstration-based approaches, this work focuses on learning agentic policy from action guidance, with minimal intervention that adopts to tasks of different difficulty.

%% file: sec/5_conclusion.tex
\section{Conclusion}
\vspace{-5pt}

We present \textsc{ActGuide-RL}, a framework that leverages readily available action data as plan-style guidance to help agentic RL overcome exploration barriers beyond the base policy's reachable region.
By introducing guidance only as an adaptive fallback and optimizing guided and unguided rollouts jointly, \textsc{ActGuide-RL} internalizes exploration gains while reducing the off-policy risks of excessive intervention.
Across search-agent benchmarks, these design choices yield consistent gains over vanilla RL and performance comparable to SFT+RL, without requiring a supervised cold start.
Further analyses show that these gains are accompanied by more effective multi-step interaction and arise from adaptive, minimally intrusive guidance rather than simply stronger intervention.
These findings suggest that scalable action-only traces can serve as a practical post-training signal for complex agentic interaction, complementing or partially replacing costly supervised demonstrations.

%% file: sec/6_appendix.tex
\section*{Appendix}

\setlength{\jot}{8pt}

\input{sec/dataset.tex}

\input{sec/experiment_details.tex}

\input{sec/theory.tex}
\input{sec/case_study.tex}

\input{sec/limitation.tex}

%% file: sec/dataset.tex
\section{Datasets}
\label{sec:datasets}

\subsection{Train}

We adopted the search agent RL training data from ASearcher~\cite{gao2025turnsunlockinglonghorizonagentic}. 
Specifically, we sampled 2k instances to be used for RL training across all our experimental settings.

Additionally, to acquire the action data, we utilized Tongyi-DeepResearch-30B-A3B~\cite{tongyidr} as the expert model to conduct rejection sampling. 
Consistent with our experimental settings, we restricted the toolset to only two types of tools: \textit{web-search} and \textit{web-visit}. 
The correct trajectories generated by the expert model were collected, from which we exclusively extracted the atomic per-step operations (i.e., the tool call names and corresponding arguments) to serve as the candidate complete action guidance trajectories for each sample instance.
The action turns statistics of the RL training data are shwon in Figure~\ref{fig:action-turns-stats}.

For the SFT data, we sampled another disjoint subset from the ASearcher dataset in a similar manner. 
We also employed Tongyi-DeepResearch-30B-A3B to conduct rejection sampling, yielding 4k complete search agent trajectories. 
Unlike the action data, the SFT data inherently preserves the comprehensive elements of the trajectory, emphasizing the retention of the full Chain-of-Thought (CoT) reasoning, explicit tool calls, and corresponding tool responses.

\subsection{Evaluation}

To comprehensively evaluate our proposed search agent's capabilities in complex reasoning and deep search, we adopt several standard and challenging deep search benchmarks. The details of the utilized datasets are outlined below:

\begin{itemize}[leftmargin=16pt,itemsep=2pt,topsep=0pt,parsep=0pt]
\item \textbf{GAIA}~\citep{mialon_gaia_2023} is a challenging general AI assistant benchmark comprising real-world questions that require deep reasoning and web browsing. Following previous works, we utilize a subset of 103 text-only questions to test the fundamental capabilities of our system.

\item \textbf{WebWalkerQA}~\citep{wu_webwalker_2025} evaluates LLMs in complex web traversal and information gathering. It contains 680 QA tasks requiring agents to systematically traverse multiple dynamic web pages to discover multi-layered information via multi-hop reasoning.

\item \textbf{XBench}~\citep{chen_xbench_2025} specifically assesses the deep search capabilities of AI agents. It comprises 100 questions and dynamically evaluates high-order information retrieval and tool usage abilities across real-world scenarios, considering both search breadth and reasoning depth.

\item \textbf{BrowseComp-ZH}~\citep{zhou2025browsecompzh} is a complex benchmark measuring web browsing and reasoning within the Chinese internet ecosystem. It comprises 289 native, multi-hop retrieval questions strictly cross-validated across major search engines to test sophisticated multi-step reasoning.
\end{itemize}

To further assess out-of-domain generalization beyond the search-agent setting, we also evaluate on three general-purpose benchmarks:

\begin{itemize}[leftmargin=16pt,itemsep=2pt,topsep=0pt,parsep=0pt]

\item \textbf{GPQA}~\citep{rein2023gpqa} is a graduate-level, Google-proof question-answering benchmark covering difficult scientific domains. We use it as an out-of-domain reasoning benchmark beyond the search-agent setting.

\item \textbf{TruthfulQA}~\citep{lin2022truthfulqa} evaluates whether language models generate truthful answers rather than imitating common misconceptions, providing an out-of-domain test of factual robustness.

\item \textbf{IFEval}~\citep{zhou2023instruction} measures instruction-following ability with verifiable constraints, serving as an out-of-domain benchmark for general alignment and controllability.
\end{itemize}

%% file: sec/experiment_details.tex
\section{Experiment Details}
\label{sec:experiment details}

\noindent\textbf{Implementation Details.}
Our implementation is built upon VeRL.
All the experimental hyperparameter settings are listed in Table~\ref{tab:hyperparameters}.
During guided rollout, we inject the action data into the query prompt as plan-style reference guidance, so that the policy can follow the partial action trajectory while still completing any missing steps by itself.
The exact prompt format is shown in Template~\ref{box:ActGuide prompt}.

\begin{figure}[htbp]
    \centering
    \begin{minipage}[t]{0.40\textwidth}
        \vspace{0pt}
        \centering
        \captionof{table}{
            Hyperparameters for search-agent RL training.
        }
        \label{tab:hyperparameters}
        \resizebox{\linewidth}{!}{
        \begin{tabular}{lc}
            \toprule
            Config & Setting \\
            \midrule
            optimizer & AdamW \\
            learning rate & 1e-6 \\
            KL coefficient & 0.001 \\
            training data & 2,000 \\
            total training steps & 64 \\
            training batch size & 32 \\
            PPO mini batch size & 16 \\
            group size & 8 \\
            max response length & 40,960 \\
            max observation length & 8,000 \\
            max turns & 30 \\
            $\epsilon_{\text{clip}_\text{low}}$ & 0.2 \\
            $\epsilon_{\text{clip}_\text{high}}$ & 0.2 \\
            \bottomrule
        \end{tabular}
        }
    \end{minipage}
    \hfill
    \begin{minipage}[t]{0.40\textwidth}
        \vspace{0pt}
        \centering
        \includegraphics[width=\linewidth]{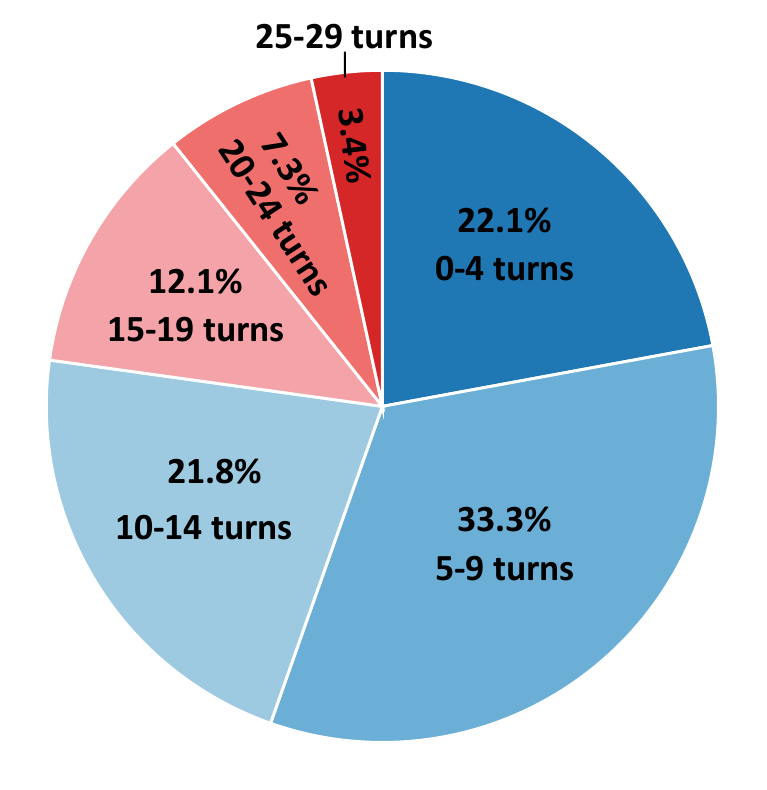}
        \caption{
            Action turns statistics of RL training data.
        }
        \label{fig:action-turns-stats}
    \end{minipage}
\end{figure}

\phantomsection
\label{box:ActGuide prompt}
\begin{rawtemplatebox}{\textsc{ActGuide} Prompt Template}
Answer the given question using the given tools.
For each step, you must conduct a thought section to reason before calling any tools.

Question: {question}

Follow the partial action trajectory hint to take actions. Note that the trajectory may not be complete, and you may still need to make extra tool calls to finish the task.

Reference action trajectory hint: {action trajectory}
\end{rawtemplatebox}

\phantomsection
\label{box:llm judge prompt}
\begin{rawtemplatebox}{LLM-Judge Prompt Template}
Based on the given question, standard answer, and model-predicted answer,
evaluate whether the model's response is correct. Your task is to classify
the result as: [CORRECT] or [INCORRECT].

First, we'll list examples for each category, then you'll evaluate a new
question's predicted answer.

Here are examples of [CORRECT] responses:
```
Question: What are the names of Barack Obama's children?
Standard Answer: Malia Obama and Sasha Obama
Model Prediction 1: Malia Obama and Sasha Obama
Model Prediction 2: Malia and Sasha
Model Prediction 3: Most would say Malia and Sasha, but I'm not sure, I should verify
Model Prediction 4: Barack Obama has two daughters, Malia Ann and Natasha Marian,
                    commonly known as Malia Obama and Sasha Obama.
```

These responses are all [CORRECT] because they:
  - Fully include the important information from the standard answer.
  - Don't contain any information that contradicts the standard answer.
  - Focus only on semantic content; language, capitalization, punctuation,
    grammar, and order aren't important.
  - Vague statements or guesses are acceptable as long as they include the
    standard answer and don't contain incorrect information or contradictions.

Here are examples of [INCORRECT] responses:
```
Question: What are the names of Barack Obama's children?
Standard Answer: Malia Obama and Sasha Obama
Model Prediction 1: Malia
Model Prediction 2: Malia, Sasha and Susan or Sasha Obama or Malia Obama,
                    or Natasha Marian, or Einstein
Model Prediction 3: While I don't know their exact names, I can tell you
                    Barack Obama has two children.
Model Prediction 4: You might be thinking of Betsy and Olivia. But you should
                    verify the details with the latest references. Is that
                    the correct answer?
Model Prediction 5: Barack Obama's children
```

These responses are all [INCORRECT] because they:
  - Contain factual statements that contradict the standard answer.
  - Are empty or merely repeat the question.
  - Enumerate multiple answers or repeat the answer.

Pay special attention to the following:
  - The standard answer may contain responses to multiple aspects of the
    question. Within the same aspect, there may be different descriptions,
    all of which are correct and are given in the same bracket, connected by
    commas. For example:
      Question: What is the name of ByteDance's AI model?
      Standard Answer: [[Doubao, Skylark]]
      Correct Predictions: Doubao; Doubao, Skylark; Skylark

  - For standard answers containing responses to different aspects, the model
    needs to provide answers to all aspects to be considered correct. There is
    no [PARTIALLY CORRECT] output option. These answers will be given in
    different brackets. For example:
      Question: Who are the members of TFBOYS?
      Standard Answer: [[Wang Junkai][Wang Yuan][Yi Yangqianxi]]
      Correct Prediction: Wang Junkai, Wang Yuan, Yi Yangqianxi
      Incorrect Prediction: Wang Junkai, Yi Yangqianxi

Also note the following points:
  - For questions with numerical standard answers, the predicted answer should
    match the standard answer. For example:
      Question: What is the total length in meters of the Huangpu River Bridge
                on the Jinshan Railway?
      Standard Answer: 3518.17
      Correct Predictions: 3518; 3518.1; 3518.17
      Incorrect Predictions: 3520; 3600

  - If the model prediction doesn't directly answer the question, attempts to
    circumvent, or fails to directly provide the standard answer, it's
    considered an [INCORRECT] answer. For example:
      Question: Who is JJ Lin's wife?
      Standard Answer: Ding Wenqi
      Incorrect Predictions: JJ Lin's wife; JJ Lin's wife should be excellent;
                             JJ Lin's wife might be a public figure

  - If the standard answer contains more information than the question asks
    for, the predicted answer only needs to include the information mentioned
    in the question. For example:
      Question: What is the main chemical component of magnesite?
      Standard Answer: Magnesium carbonate (MgCO3)
      Correct Predictions: Magnesium carbonate; MgCO3

  - If information omitted in the predicted answer can be clearly inferred
    from the question, it's considered correct. For example:
      Question: The Nuragic ruins of Barumini were listed as a World Cultural
                Heritage by UNESCO in 1997, so where is this site located?
      Standard Answer: Sardinia, Italy
      Correct Prediction: Sardinia

  - If it's clear that different translations of a name refer to the same
    person, it's considered correct. For example, if the standard answer is
    Robinson, answers like Lubinson or Lubinsun are both correct.

  - Focus more on the match between the standard answer and the model
    prediction than whether the standard answer itself is correct.

Below is a new question example. Please reply with only [CORRECT] or
[INCORRECT], without apologies or corrections to your own errors. Just
evaluate the answer.
```
Question: {question}
Standard Answer: {correct_answer}
Predicted Answer: {response}
```

Evaluate this new question's predicted answer as one of the following:
A. [CORRECT]
B. [INCORRECT]

Return only the option representing [CORRECT] or [INCORRECT], i.e., just
return A or B, without adding any other text.
\end{rawtemplatebox}


\noindent\textbf{Compute Resources.}
All training and rollout experiments were conducted on nodes equipped with 8 NVIDIA H20 GPUs.
The LLM judge used for reward assignment and test-time evaluation required additional serving resources, for which we used a separate node with 8 NVIDIA H20 GPUs.

\noindent\textbf{Different Guidance Methods.}
We also compare different ways of injecting action guidance for LLM-based agent.
Besides the unguided setting, we consider an assistant-prefix format following prior hint-based methods, where the action reference is prepended as a generated prefix and the model continues generation from it.
We also consider a user-assistant message format, where the action data are converted into the corresponding tool calls and tool responses and then assembled as multi-turn messages before the model continues generation.
As shown in Table~\ref{tab:inject-method}, inserting the action trajectory as a reference plan in the query prompt achieves the best Reward@1, suggesting that lightweight plan-style guidance is more effective than directly prefixing or replaying actions for LLM agent.

\begin{table}[htbp]
    \centering
    \caption{
        Injection method comparison.
    }
    \label{tab:inject-method}
    \resizebox{\textwidth}{!}{
    \begin{tabular}{lcccc}
        \toprule
        \textbf{Inject Method} & \textbf{Unguidance} & \textbf{Assistant Prefix} & \textbf{User-Assistant Messages} & \textbf{Reference Plan in Query Prompt} \\
        \midrule
        \textbf{Reward@1} & 57.90 & 74.50 & 80.10 & \textbf{85.70} \\
        \bottomrule
    \end{tabular}
    }
\end{table}

\noindent\textbf{Action Data for On-policy Self Distillation.}
Beyond using action data to guide the policy toward better state visitation, we also explore whether it can be used for on-policy self distillation (OPSD)~\citep{zhao2026self,shenfeld_self-distillation_2026,hubotter2026reinforcement}.
Specifically, OPSD still samples trajectories from the unguided policy, but uses action-conditioned guided logits as the distillation target on these on-policy rollouts.
Formally, for an unguided rollout $\tau \sim \pi_{\theta_{\rm old}}(\cdot \mid x)$, we re-evaluate each visited prefix $z_{<t}$ with the same model additionally conditioned on the action guidance $g$, and optimize
\begin{equation}
    \mathcal{L}_{\mathrm{OPSD}}(\theta)
    =
    \mathbb{E}_{x \sim \mathcal{D},\, \tau \sim \pi_{\theta_{\rm old}}(\cdot \mid x)}
    \left[
    \frac{1}{T}\sum_{t=1}^{T}
    \mathbb{D}_{\mathrm{KL}}\!\left(
    \mathrm{sg}\!\left[\pi_{\theta_{\rm old}}(\cdot \mid z_{<t}, g)\right]
    \,\|\, \pi_{\theta}(\cdot \mid z_{<t})
    \right)
    \right],
\end{equation}
where $\mathrm{sg}[\cdot]$ denotes stop-gradient, so the guided distribution serves only as a token-level teacher while the learned policy remains unguided at inference time.
As shown in Table~\ref{tab:guidance-use}, OPSD can improve model performance, but the gains remain limited because the visited states are still determined by the base unguided policy.
Therefore, it does not fundamentally resolve the ineffective state-visitation problem when the agent cannot reach useful states by itself.

\begin{table}[htbp]
    \centering
    \caption{
        Comparison between \textsc{ActGuide-RL} and OPSD.
    }
    \label{tab:guidance-use}
    \begin{tabular}{lccc}
        \toprule
        \textbf{Guidance Use} & \textbf{GAIA} & \textbf{WebWalker} & \textbf{XBench} \\
        \midrule
        \textbf{\textsc{ActGuide-RL}} & 35.92 & \textbf{39.85} & \textbf{37.00} \\
        OPSD & \textbf{36.89} & 30.29 & 26.00 \\
        \bottomrule
    \end{tabular}
\end{table}

%% file: sec/theory.tex
\section{Theoretical Analysis}
\label{sec:theory}

\subsection{Covariance Form of the Token-Level Off-Policy Risk}

In Section~\ref{sec:minimal_intervention}, let $\tau=(z_1,\ldots,z_{|\tau|})$ be the generated token sequence.
We define the token-level importance ratio under guidance level $g_k$ as
\begin{equation}
    r_j^{(k)}
    :=
    \frac{\pi_\theta(z_j \mid z_{<j})}
         {\pi_\theta(z_j \mid z_{<j}, g_k)},
\end{equation}
and the corresponding cumulative log-ratio shift as
\begin{equation}
    \mathcal{L}_k(\tau)
    :=
    \sum_{j=1}^{|\tau|} \log r_j^{(k)}.
\end{equation}
The off-policy risk is then defined as
\begin{equation}
    R_k
    :=
    \mathrm{Var}_{\tau \sim \pi_\theta(\cdot \mid s, g_k)}\!\left(\mathcal{L}_k(\tau)\right).
\end{equation}

By variance expansion, we have
\begin{equation}
    R_k
    =
    \mathrm{Var}_{\tau \sim \pi_\theta(\cdot \mid s, g_k)}
    \left(
        \sum_{j=1}^{|\tau|} \log r_j^{(k)}
    \right).
\end{equation}
Therefore,
\begin{equation}
    R_k
    =
    \sum_{j=1}^{|\tau|}
    \mathrm{Var}_{\tau \sim \pi_\theta(\cdot \mid s, g_k)}\!\left(\log r_j^{(k)}\right)
    +
    2
    \sum_{j<j'}
    \mathrm{Cov}_{\tau \sim \pi_\theta(\cdot \mid s, g_k)}
    \!\left(
        \log r_j^{(k)},
        \log r_{j'}^{(k)}
    \right),
\end{equation}
where the second summation ranges over all distinct token pairs in the rollout.

This decomposition shows that the off-policy risk consists of two components:
(1) token-wise variance terms, which capture local distribution mismatch at each generation step, and
(2) cross-token covariance terms, which capture the dependence structure of these mismatches along the autoregressive trajectory.
Hence, stronger guidance may increase not only the magnitude of token-level deviations, but also their correlation across the rollout, both of which contribute to larger internalization risk.

In particular, if the token-level log-ratio shifts were independent, then all covariance terms would vanish and $R_k$ would reduce to the sum of token-wise variances.
However, in autoregressive agent generation, token dependencies are intrinsic, and the covariance terms generally cannot be ignored.
This motivates using the variance of the cumulative log-ratio shift as a compact measure of off-policy risk.

\subsection{Risk-Constrained View of Minimal Intervention}

We formalize the minimal-intervention rule in Section~\ref{sec:minimal_intervention} as a risk-constrained selection problem.
For a guidance level $g_k$, define the group recovery probability
\begin{equation}
    Q_k
    :=
    \mathbb{P}_{\{\tau_i^{(k)}\}_{i=1}^N \sim \pi_\theta(\cdot\mid x,g_k)}
    \!\left(
        \max_{i\le N}Y(\tau_i^{(k)}) \ge \delta
    \right),
\end{equation}
where $N$ is the rollout group size and $\delta$ is the success threshold.
Given a target recovery level $\rho\in(0,1)$, the least-risk sufficient guidance level is the solution of
\begin{equation}
    \min_{k\in\{0,\ldots,K\}} R_k
    \quad
    \mathrm{s.t.}
    \quad
    Q_k \ge \rho .
    \label{eq:risk_constrained_guidance}
\end{equation}

\begin{assumption}[Monotone recovery and risk]\label{assump:monotone_guidance}
    The ordered guidance family $g_0\prec g_1\prec\cdots\prec g_K$ satisfies:
    \begin{equation}
        Q_0 \le Q_1 \le \cdots \le Q_K,
        \qquad
        R_0 \le R_1 \le \cdots \le R_K .
    \end{equation}
\end{assumption}

\begin{proposition}[Minimal sufficient guidance is risk-optimal]\label{prop:risk_minimal_guidance}
    Under Assumption~\ref{assump:monotone_guidance}, if the feasible set of Eq.~\ref{eq:risk_constrained_guidance} is non-empty, then
    \begin{equation}
        k_\rho^\star
        :=
        \min\{k\in\{0,\ldots,K\}: Q_k\ge\rho\}
    \end{equation}
    is an optimal solution of Eq.~\ref{eq:risk_constrained_guidance}.
\end{proposition}

\begin{proof}
    By definition, $k_\rho^\star$ is feasible.
    For any other feasible level $k$, minimality of $k_\rho^\star$ implies $k\ge k_\rho^\star$.
    Since $R_k$ is non-decreasing in $k$, we have $R_k\ge R_{k_\rho^\star}$.
    Hence no feasible guidance level has smaller off-policy risk than $k_\rho^\star$.
\end{proof}

This proposition gives a constrained interpretation of Eq.~\ref{eq:binary_search}: minimal intervention does not maximize guidance strength, but selects the lowest-risk level that satisfies a recovery requirement.
When $Q_k$ is not known exactly, it can be estimated by repeated rollout groups.
Let $\widehat Q_k$ be the empirical mean of $m$ independent group-recovery indicators at level $k$.

\begin{corollary}[Empirical identification under a margin]\label{cor:empirical_guidance}
    Suppose Assumption~\ref{assump:monotone_guidance} holds and there exists a margin $\Delta>0$ such that
    \begin{equation}
        Q_k \le \rho-\Delta \quad \forall k<k_\rho^\star,
        \qquad
        Q_k \ge \rho+\Delta \quad \forall k\ge k_\rho^\star .
    \end{equation}
    If
    \begin{equation}
        m \ge \frac{1}{2\Delta^2}\log\frac{2(K+1)}{\xi},
    \end{equation}
    then with probability at least $1-\xi$, the empirical rule
    \begin{equation}
        \widehat k_\rho
        :=
        \min\{k:\widehat Q_k\ge\rho\}
    \end{equation}
    recovers $k_\rho^\star$.
\end{corollary}

\begin{proof}
    By Hoeffding's inequality and a union bound over $K+1$ levels,
    \begin{equation}
        \mathbb{P}\!\left(\max_k|\widehat Q_k-Q_k|\ge\Delta\right)
        \le
        2(K+1)\exp(-2m\Delta^2)
        \le \xi .
    \end{equation}
    On the complementary event, every $k<k_\rho^\star$ has $\widehat Q_k<\rho$, while every $k\ge k_\rho^\star$ has $\widehat Q_k\ge\rho$.
    Therefore the empirical minimal feasible level equals $k_\rho^\star$.
\end{proof}

%% file: sec/case_study.tex
\section{Training Cases of \textsc{ActGuide-RL}}
\label{sec:case study}

We present representative training cases to illustrate how \textsc{ActGuide-RL} injects action data as plan-style guidance.
Each case uses the full guidance level $g_K$, where the prompt contains the complete available reference action trajectory while still requiring the policy to reason and complete the task with tools.

\begin{actguidecasebox}{Training Case 1 of \textsc{ActGuide-RL} with Full Guidance $g_K$}
    \label{case:case study 1}
    \actguidecaseheader{Example from RL Training Data}

    \noindent\textbf{Original Question}: Which political party last held overall control of the West Berkshire Council before the election that first introduced the ward named after the Berkshire village where the Kennet \& Avon Canal route was restored to Hungerford Wharf in July 1974, resulting in a tied outcome necessitating by-elections and enabling the Conservatives to secure a majority?

    \noindent\textbf{Ground Truth}: The Liberal Democrats

    \actguidecaseheader{Prompt for \textsc{ActGuide-RL}}

    \noindent Answer the given question using the given tools.
    For each step, you must conduct a thought section to reason before calling any tools.

    \noindent Question: Which political party last held overall control of the West Berkshire Council before the election that first introduced the ward named after the Berkshire village where the Kennet \& Avon Canal route was restored to Hungerford Wharf in July 1974, resulting in a tied outcome necessitating by-elections and enabling the Conservatives to secure a majority?

    \noindent Follow the partial action trajectory hint to take actions, note that the trajectory may not complete and you still need do some extra tool calls to finish the task.

    \noindent Reference action trajectory hint:

    \noindent \textless 1\textgreater: search [``Kennet \& Avon Canal restored to Hungerford Wharf July 1974'', ``Hungerford Wharf July 1974 canal restoration'', ``Kennet and Avon canal restoration July 1974'', ``Berkshire village Hungerford Wharf restored July 1974'']
    \noindent \textless 2\textgreater: search [``West Berkshire Council election tied result by-elections Conservatives secured majority new ward named after Hungerford'', ``West Berkshire Council election tied outcome by-elections Conservative majority ward introduction'', ``West Berkshire Council ward first introduced election tied by-elections'']
    \noindent \textless 3\textgreater: visit [\url{https://en.wikipedia.org/wiki/West_Berkshire_Council_elections}]
    \noindent \textless 4\textgreater: search [``Hungerford ward West Berkshire Council'', ``Hungerford electoral ward West Berkshire'', ``Hungerford West Berkshire Council ward'']
    \noindent \textless 5\textgreater: search [``tied West Berkshire council election by-election Conservative majority'', ``tied West Berkshire by-election'', ``by-elections West Berkshire Conservative majority'']
    \noindent \textless 6\textgreater: search [``West Berkshire Council tied election'', ``West Berkshire tied election'', ``West Berkshire tied council'']
    \noindent \textless 7\textgreater: search [``tied result West Berkshire council by-election'', ``tied by-elections West Berkshire'', ``tie West Berkshire council election'']
    \noindent \textless 8\textgreater: search [``West Berkshire 2003 new ward'', ``West Berkshire 2003 ward boundary'', ``West Berkshire 2003 electoral changes'']
    \noindent \textless 9\textgreater: visit [\url{https://grokipedia.com/page/2003_west_berkshire_council_election}]
    \noindent \textless 10\textgreater: search [``Hungerford ward West Berkshire 2003'', ``Hungerford West Berkshire new ward 2003'']
    \noindent \textless 11\textgreater: visit [\url{https://www.andrewteale.me.uk/leap/ward/2226/}]
    \noindent \textless 12\textgreater: visit [\url{https://en.wikipedia.org/wiki/West_Berkshire_Council_elections}]
\end{actguidecasebox}

\begin{actguidecasebox}{Training Case 2 of \textsc{ActGuide-RL} with Full Guidance $g_K$}
    \label{case:case study 2}
    \actguidecaseheader{Example from RL Training Data}

    \noindent\textbf{Original Question}: In a chess tournament, all players faced each other twice and scored 1 point for a win, half a point for a draw, and zero points for a loss. The winner was the one who scored the most points. Lu\'iza, a curious mathematician, found a paper stating that the sum of the scores of all participants was 210 points. a) How many players participated in the tournament? b) Lu\'iza noticed that the tournament logo depicted a knight in the center, with black squares illustrating the knight's possible moves. She wondered: If the board were infinite and the knight continued its pattern of movement, onto how many different squares could the knight move in two moves? Initially, she thought ``64'' but quickly realized she was wrong. Now, it's your turn to correctly answer the question Lu\'iza thought of. c) Lu\'iza received a message from a friend who participated in the tournament, informing her that he scored 12 points. Lu\'iza replied: ``- Don't be sad, you will win the next championship!'' How did Lu\'iza know that her friend did not win the tournament?

    \noindent\textbf{Ground Truth}: 33

    \actguidecaseheader{Prompt for \textsc{ActGuide-RL}}

    \noindent Answer the given question using the given tools.
    For each step, you must conduct a thought section to reason before calling any tools.

    \noindent Question: In a chess tournament, all players faced each other twice and scored 1 point for a win, half a point for a draw, and zero points for a loss. The winner was the one who scored the most points. Lu\'iza, a curious mathematician, found a paper stating that the sum of the scores of all participants was 210 points. a) How many players participated in the tournament? b) Lu\'iza noticed that the tournament logo depicted a knight in the center, with black squares illustrating the knight's possible moves. She wondered: If the board were infinite and the knight continued its pattern of movement, onto how many different squares could the knight move in two moves? Initially, she thought ``64'' but quickly realized she was wrong. Now, it's your turn to correctly answer the question Lu\'iza thought of. c) Lu\'iza received a message from a friend who participated in the tournament, informing her that he scored 12 points. Lu\'iza replied: ``- Don't be sad, you will win the next championship!'' How did Lu\'iza know that her friend did not win the tournament?

    \noindent Follow the partial action trajectory hint to take actions, note that the trajectory may not complete and you still need do some extra tool calls to finish the task.

    \noindent Reference action trajectory hint:

    \noindent \textless 1\textgreater: search [``knight moves two moves reachable squares infinite board 33'']
    \noindent \textless 2\textgreater: search [``knight two moves reachable squares infinite board'']
    \noindent \textless 3\textgreater: visit [\url{https://www.reddit.com/r/chess/comments/bem4zw/given_an_infinite_chess_board_how_many_spaces/}]
    \noindent \textless 4\textgreater: search [``how many squares can a knight reach in two moves'']
    \noindent \textless 5\textgreater: search [``knight reach two moves how many squares'']
    \noindent \textless 6\textgreater: visit [\url{https://file.scirp.org/Html/8-1200157_34513.htm}]
    \noindent \textless 7\textgreater: search [``33 squares knight two moves'']
    \noindent \textless 8\textgreater: visit [\url{https://www.mathworks.com/matlabcentral/cody/problems/60541-compute-the-number-of-squares-a-knight-can-reach-after-n-moves}]
    \noindent \textless 9\textgreater: visit [\url{https://www.mathworks.com/content/dam/mathworks/images/cody/graphics/ChessKnight.svg}]
    \noindent \textless 10\textgreater: search [``33 squares knight two moves infinite board'']
    \noindent \textless 11\textgreater: search [``33 squares after two moves knight'']
    \noindent \textless 12\textgreater: search [``knight can reach after two moves 33'']
    \noindent \textless 13\textgreater: visit [\url{https://www.mathworks.com/matlabcentral/cody/problems/60541-compute-the-number-of-squares-a-knight-can-reach-after-n-moves}]
    \noindent \textless 14\textgreater: search [``knight two moves 33 squares'']
    \noindent \textless 15\textgreater: visit [\url{https://www.reddit.com/r/chess/comments/1j3roeo/the_amount_of_space_a_single_knight_can_control/}]
\end{actguidecasebox}

\begin{actguidecasebox}{Training Case 3 of \textsc{ActGuide-RL} with Full Guidance $g_K$}
    \label{case:case study 3}
    \actguidecaseheader{Example from RL Training Data}

    \noindent\textbf{Original Question}: In which year, during the early 1970s, did a player, known for his time with the New York Yankees, join the team, and who is the former MLB pitcher, now a pitching coach for a minor league team affiliated with the Miami Marlins, whose first name starts with `M' and who started the opening game of the 2009 World Baseball Classic against Venezuela, pitching 4 shutout innings?

    \noindent\textbf{Ground Truth}: 1972, Mark DiFelice

    \actguidecaseheader{Prompt for \textsc{ActGuide-RL}}

    \noindent Answer the given question using the given tools.
    For each step, you must conduct a thought section to reason before calling any tools.

    \noindent Question: In which year, during the early 1970s, did a player, known for his time with the New York Yankees, join the team, and who is the former MLB pitcher, now a pitching coach for a minor league team affiliated with the Miami Marlins, whose first name starts with `M' and who started the opening game of the 2009 World Baseball Classic against Venezuela, pitching 4 shutout innings?

    \noindent Follow the partial action trajectory hint to take actions, note that the trajectory may not complete and you still need do some extra tool calls to finish the task.

    \noindent Reference action trajectory hint:

    \noindent \textless 1\textgreater: search [``2009 World Baseball Classic opening game Venezuela 4 shutout innings pitcher'', ``Venezuela 2009 WBC opening game pitcher 4 shutout innings'', ``World Baseball Classic 2009 Venezuela opener pitcher'', ``2009 WBC Venezuela vs Opening Game pitcher'', ``2009 WBC Venezuela 4 shutout innings start'']
    \noindent \textless 2\textgreater: search [``2009 World Baseball Classic opening game March 5 2009 opening games schedule'', ``first game of 2009 World Baseball Classic opening day'']
    \noindent \textless 3\textgreater: visit [\url{https://www.aol.com/articles/miami-opening-game-world-baseball-183425977.html}]
    \noindent \textless 4\textgreater: search [``2009 World Baseball Classic Venezuela vs Italy opening game 4 shutout innings pitcher'', ``Venezuela vs Italy March 7 2009 WBC pitcher started 4 shutout innings'']
    \noindent \textless 5\textgreater: visit [\url{https://www.arkansasonline.com/news/2009/mar/11/world-baseball-classic-venezuelas-power-too-much-for-italy/}]
    \noindent \textless 6\textgreater: visit [\url{https://www.espn.com/world-baseball-classic/boxscore/_/gameId/290307116}]
    \noindent \textless 7\textgreater: search [``Mark DiFelice pitching coach Miami Marlins affiliate'', ``Mark DiFelice minor league pitching coach Marlins'']
    \noindent \textless 8\textgreater: search [``joined the New York Yankees in 1970'', ``joined the New York Yankees in 1971'', ``joined the New York Yankees in 1972'', ``joined the New York Yankees in 1973'', ``joined the New York Yankees in 1974'', ``joined the New York Yankees in 1975'']
    \noindent \textless 9\textgreater: search [``joined the New York Yankees 1970'', ``joined the New York Yankees 1971'', ``joined the New York Yankees 1972'', ``joined the New York Yankees 1973'', ``joined the New York Yankees 1974'', ``joined the New York Yankees 1975'']
    \noindent \textless 10\textgreater: search [``Sparky Lyle joined the New York Yankees 1972'']
    \noindent \textless 11\textgreater: search [``Graig Nettles joined the New York Yankees 1972'']
    \noindent \textless 12\textgreater: search [``known for his time with the New York Yankees'']
    \noindent \textless 13\textgreater: search [``Thurman Munson joined the New York Yankees'']
    \noindent \textless 14\textgreater: visit [\url{https://pinstripesnation.com/thurman-munson-the-captain-who-never-left-the-yankees-2023-08-02/}]
    \noindent \textless 15\textgreater: search [``Graig Nettles known for his time with the New York Yankees'']
    \noindent \textless 16\textgreater: search [``Mark DiFelice 2009 World Baseball Classic opening game Venezuela 4 shutout innings'']
    \noindent \textless 17\textgreater: visit [\url{https://en.wikipedia.org/wiki/Mark_DiFelice}]
    \noindent \textless 18\textgreater: visit [\url{https://www.ebay.com/itm/389131297484}]
    \noindent \textless 19\textgreater: visit [\url{https://en.wikipedia.org/wiki/Graig_Nettles}]
    \noindent \textless 20\textgreater: visit [\url{https://en.wikipedia.org/wiki/Mark_DiFelice}]
    \noindent \textless 21\textgreater: search [``Mark DiFelice Marlins'', ``Mark DiFelice Miami Marlins'']
    \noindent \textless 22\textgreater: visit [\url{https://en.wikipedia.org/wiki/Graig_Nettles}]
    \noindent \textless 23\textgreater: search [``Sparky Lyle New York Yankees known for his time with'']
    \noindent \textless 24\textgreater: visit [\url{https://en.wikipedia.org/wiki/Sparky_Lyle}]
    \noindent \textless 25\textgreater: search [``best known for his time with the New York Yankees Sparky Lyle'']
    \noindent \textless 26\textgreater: search [``best known for his time with the New York Yankees Graig Nettles'']
\end{actguidecasebox}

%% file: sec/limitation.tex
\section{Limitations}
\label{sec:limitations}

Due to the relatively simple experimental setup, the ease of obtaining task queries with different difficulty levels, and the natural availability of action data, our main experiments are conducted in the search-agent setting.
This setting provides a controlled testbed for studying reachability barriers and guidance-induced off-policy risk.
Nevertheless, \textsc{ActGuide-RL} is designed for general agentic training rather than being specific to search agents, and its effectiveness in other agent tasks, such as CLI, GUI, API-based, and embodied environments, remains to be further explored.

This work utilizes action data through plan-style guidance, where reference actions are injected as a high-level action plan to help the policy cross exploration barriers.
This simple formulation keeps the method lightweight, broadly applicable, and independent of costly reasoning traces.
More fine-grained ways of using action data, such as step-level guidance injection, also remain to be further explored.

This work focuses on how to leverage action data for agentic RL, but does not discuss how such data should be systematically collected and processed.
In practice, structured collection, cleaning, and filtering of existing interaction records, such as backend logs from different agent applications, are also important for action-data-based training and remain worth exploring.

%% file: ref.bib
@misc{wang_ragen_2025,
	title = {{RAGEN}: {Understanding} {Self}-{Evolution} in {LLM} {Agents} via {Multi}-{Turn} {Reinforcement} {Learning}},
	shorttitle = {{RAGEN}},
	url = {http://arxiv.org/abs/2504.20073},
	doi = {10.48550/arXiv.2504.20073},
	urldate = {2025-05-27},
	publisher = {arXiv},
	author = {Wang, Zihan and Wang, Kangrui and Wang, Qineng and Zhang, Pingyue and Li, Linjie and Yang, Zhengyuan and Jin, Xing and Yu, Kefan and Nguyen, Minh Nhat and Liu, Licheng and Gottlieb, Eli and Lu, Yiping and Cho, Kyunghyun and Wu, Jiajun and Fei-Fei, Li and Wang, Lijuan and Choi, Yejin and Li, Manling},
	month = may,
	year = {2025},
	note = {arXiv:2504.20073 [cs]},
}

@misc{jin_search-r1_2025,
	title = {Search-{R1}: {Training} {LLMs} to {Reason} and {Leverage} {Search} {Engines} with {Reinforcement} {Learning}},
	shorttitle = {Search-{R1}},
	url = {http://arxiv.org/abs/2503.09516},
	doi = {10.48550/arXiv.2503.09516},
	urldate = {2025-05-27},
	publisher = {arXiv},
	author = {Jin, Bowen and Zeng, Hansi and Yue, Zhenrui and Yoon, Jinsung and Arik, Sercan and Wang, Dong and Zamani, Hamed and Han, Jiawei},
	month = apr,
	year = {2025},
	note = {arXiv:2503.09516 [cs]},
}

@misc{yue_does_2025,
	title = {Does {Reinforcement} {Learning} {Really} {Incentivize} {Reasoning} {Capacity} in {LLMs} {Beyond} the {Base} {Model}?},
	url = {http://arxiv.org/abs/2504.13837},
	doi = {10.48550/arXiv.2504.13837},
	urldate = {2025-05-28},
	publisher = {arXiv},
	author = {Yue, Yang and Chen, Zhiqi and Lu, Rui and Zhao, Andrew and Wang, Zhaokai and Yue, Yang and Song, Shiji and Huang, Gao},
	month = may,
	year = {2025},
	note = {arXiv:2504.13837 [cs]},
}

@misc{shao_deepseekmath_2024,
	title = {{DeepSeekMath}: {Pushing} the {Limits} of {Mathematical} {Reasoning} in {Open} {Language} {Models}},
	shorttitle = {{DeepSeekMath}},
	url = {http://arxiv.org/abs/2402.03300},
	doi = {10.48550/arXiv.2402.03300},
	language = {en-US},
	urldate = {2025-05-28},
	publisher = {arXiv},
	author = {Shao, Zhihong and Wang, Peiyi and Zhu, Qihao and Xu, Runxin and Song, Junxiao and Bi, Xiao and Zhang, Haowei and Zhang, Mingchuan and Li, Y. K. and Wu, Y. and Guo, Daya},
	month = apr,
	year = {2024},
	note = {arXiv:2402.03300 [cs]},
}

@misc{feng_group--group_2025,
	title = {Group-in-{Group} {Policy} {Optimization} for {LLM} {Agent} {Training}},
	url = {http://arxiv.org/abs/2505.10978},
	doi = {10.48550/arXiv.2505.10978},
	urldate = {2025-05-28},
	publisher = {arXiv},
	author = {Feng, Lang and Xue, Zhenghai and Liu, Tingcong and An, Bo},
	month = may,
	year = {2025},
	note = {arXiv:2505.10978 [cs]},
}

@misc{li_webthinker_2025,
	title = {{WebThinker}: {Empowering} {Large} {Reasoning} {Models} with {Deep} {Research} {Capability}},
	shorttitle = {{WebThinker}},
	url = {http://arxiv.org/abs/2504.21776},
	doi = {10.48550/arXiv.2504.21776},
	urldate = {2025-06-04},
	publisher = {arXiv},
	author = {Li, Xiaoxi and Jin, Jiajie and Dong, Guanting and Qian, Hongjin and Zhu, Yutao and Wu, Yongkang and Wen, Ji-Rong and Dou, Zhicheng},
	month = apr,
	year = {2025},
	note = {arXiv:2504.21776 [cs]},
}

@misc{yu_dapo_2025,
	title = {{DAPO}: {An} {Open}-{Source} {LLM} {Reinforcement} {Learning} {System} at {Scale}},
	shorttitle = {{DAPO}},
	url = {http://arxiv.org/abs/2503.14476},
	doi = {10.48550/arXiv.2503.14476},
	urldate = {2025-07-24},
	publisher = {arXiv},
	author = {Yu, Qiying and Zhang, Zheng and Zhu, Ruofei and Yuan, Yufeng and Zuo, Xiaochen and Yue, Yu and Dai, Weinan and Fan, Tiantian and Liu, Gaohong and Liu, Lingjun and Liu, Xin and Lin, Haibin and Lin, Zhiqi and Ma, Bole and Sheng, Guangming and Tong, Yuxuan and Zhang, Chi and Zhang, Mofan and Zhang, Wang and Zhu, Hang and Zhu, Jinhua and Chen, Jiaze and Chen, Jiangjie and Wang, Chengyi and Yu, Hongli and Song, Yuxuan and Wei, Xiangpeng and Zhou, Hao and Liu, Jingjing and Ma, Wei-Ying and Zhang, Ya-Qin and Yan, Lin and Qiao, Mu and Wu, Yonghui and Wang, Mingxuan},
	month = may,
	year = {2025},
	note = {arXiv:2503.14476 [cs]},
}

@misc{dong_agentic_2025,
	title = {Agentic {Reinforced} {Policy} {Optimization}},
	url = {http://arxiv.org/abs/2507.19849},
	doi = {10.48550/arXiv.2507.19849},
	urldate = {2025-07-29},
	publisher = {arXiv},
	author = {Dong, Guanting and Mao, Hangyu and Ma, Kai and Bao, Licheng and Chen, Yifei and Wang, Zhongyuan and Chen, Zhongxia and Du, Jiazhen and Wang, Huiyang and Zhang, Fuzheng and Zhou, Guorui and Zhu, Yutao and Wen, Ji-Rong and Dou, Zhicheng},
	month = jul,
	year = {2025},
	note = {arXiv:2507.19849 [cs]},
}

@misc{li_websailor_2025,
	title = {{WebSailor}: {Navigating} {Super}-human {Reasoning} for {Web} {Agent}},
	shorttitle = {{WebSailor}},
	url = {http://arxiv.org/abs/2507.02592},
	doi = {10.48550/arXiv.2507.02592},
	urldate = {2025-07-31},
	publisher = {arXiv},
	author = {Li, Kuan and Zhang, Zhongwang and Yin, Huifeng and Zhang, Liwen and Ou, Litu and Wu, Jialong and Yin, Wenbiao and Li, Baixuan and Tao, Zhengwei and Wang, Xinyu and Shen, Weizhou and Zhang, Junkai and Zhang, Dingchu and Wu, Xixi and Jiang, Yong and Yan, Ming and Xie, Pengjun and Huang, Fei and Zhou, Jingren},
	month = jul,
	year = {2025},
	note = {arXiv:2507.02592 [cs]},
}

@inproceedings{
chu2026gpg,
title={{GPG}: A Simple and Strong Reinforcement Learning Baseline for Model Reasoning},
author={Xiangxiang Chu and Hailang Huang and Xiao Zhang and Fei Wei and Yong Wang},
booktitle={The Fourteenth International Conference on Learning Representations},
year={2026},
url={https://openreview.net/forum?id=inccdtfx8x}
}

@misc{mialon_gaia_2023,
	title = {{GAIA}: a benchmark for {General} {AI} {Assistants}},
	shorttitle = {{GAIA}},
	url = {http://arxiv.org/abs/2311.12983},
	doi = {10.48550/arXiv.2311.12983},
	urldate = {2025-08-23},
	publisher = {arXiv},
	author = {Mialon, Grégoire and Fourrier, Clémentine and Swift, Craig and Wolf, Thomas and LeCun, Yann and Scialom, Thomas},
	month = nov,
	year = {2023},
	note = {arXiv:2311.12983 [cs]},
}

@misc{wu_webwalker_2025,
	title = {{WebWalker}: {Benchmarking} {LLMs} in {Web} {Traversal}},
	shorttitle = {{WebWalker}},
	url = {http://arxiv.org/abs/2501.07572},
	doi = {10.48550/arXiv.2501.07572},
	urldate = {2025-08-23},
	publisher = {arXiv},
	author = {Wu, Jialong and Yin, Wenbiao and Jiang, Yong and Wang, Zhenglin and Xi, Zekun and Fang, Runnan and Zhang, Linhai and He, Yulan and Zhou, Deyu and Xie, Pengjun and Huang, Fei},
	month = aug,
	year = {2025},
	note = {arXiv:2501.07572 [cs]},
}

@misc{chen_xbench_2025,
	title = {xbench: {Tracking} {Agents} {Productivity} {Scaling} with {Profession}-{Aligned} {Real}-{World} {Evaluations}},
	shorttitle = {xbench},
	url = {http://arxiv.org/abs/2506.13651},
	doi = {10.48550/arXiv.2506.13651},
	urldate = {2025-08-23},
	publisher = {arXiv},
	author = {Chen, Kaiyuan and Ren, Yixin and Liu, Yang and Hu, Xiaobo and Tian, Haotong and Xie, Tianbao and Liu, Fangfu and Zhang, Haoye and Liu, Hongzhang and Gong, Yuan and Sun, Chen and Hou, Han and Yang, Hui and Pan, James and Lou, Jianan and Mao, Jiayi and Liu, Jizheng and Li, Jinpeng and Liu, Kangyi and Liu, Kenkun and Wang, Rui and Li, Run and Niu, Tong and Zhang, Wenlong and Yan, Wenqi and Wang, Xuanzheng and Zhang, Yuchen and Hung, Yi-Hsin and Jiang, Yuan and Liu, Zexuan and Yin, Zihan and Ma, Zijian and Mo, Zhiwen},
	month = jun,
	year = {2025},
	note = {arXiv:2506.13651 [cs]},
}

@misc{yan_learning_2025,
	title = {Learning to {Reason} under {Off}-{Policy} {Guidance}},
	url = {http://arxiv.org/abs/2504.14945},
	doi = {10.48550/arXiv.2504.14945},
	language = {en-US},
	urldate = {2026-03-03},
	publisher = {arXiv},
	author = {Yan, Jianhao and Li, Yafu and Hu, Zican and Wang, Zhi and Cui, Ganqu and Qu, Xiaoye and Cheng, Yu and Zhang, Yue},
	month = jun,
	year = {2025},
	note = {arXiv:2504.14945 [cs]},
}

@misc{shenfeld_self-distillation_2026,
	title = {Self-{Distillation} {Enables} {Continual} {Learning}},
	url = {http://arxiv.org/abs/2601.19897},
	doi = {10.48550/arXiv.2601.19897},
	urldate = {2026-03-04},
	publisher = {arXiv},
	author = {Shenfeld, Idan and Damani, Mehul and Hübotter, Jonas and Agrawal, Pulkit},
	month = jan,
	year = {2026},
	note = {arXiv:2601.19897 [cs]},
}

@article{team2026kimi,
  title={Kimi K2. 5: Visual Agentic Intelligence},
  author={Team, Kimi and Bai, Tongtong and Bai, Yifan and Bao, Yiping and Cai, SH and Cao, Yuan and Charles, Y and Che, HS and Chen, Cheng and Chen, Guanduo and others},
  journal={arXiv preprint arXiv:2602.02276},
  year={2026}
}

@article{wei2026agentic,
  title={Agentic reasoning for large language models},
  author={Wei, Tianxin and Li, Ting-Wei and Liu, Zhining and Ning, Xuying and Yang, Ze and Zou, Jiaru and Zeng, Zhichen and Qiu, Ruizhong and Lin, Xiao and Fu, Dongqi and others},
  journal={arXiv preprint arXiv:2601.12538},
  year={2026}
}

@article{luo2025large,
  title={Large language model agent: A survey on methodology, applications and challenges},
  author={Luo, Junyu and Zhang, Weizhi and Yuan, Ye and Zhao, Yusheng and Yang, Junwei and Gu, Yiyang and Wu, Bohan and Chen, Binqi and Qiao, Ziyue and Long, Qingqing and others},
  journal={arXiv preprint arXiv:2503.21460},
  year={2025}
}

@inproceedings{yao2022react,
  title={React: Synergizing reasoning and acting in language models},
  author={Yao, Shunyu and Zhao, Jeffrey and Yu, Dian and Du, Nan and Shafran, Izhak and Narasimhan, Karthik R and Cao, Yuan},
  booktitle={The eleventh international conference on learning representations},
  year={2022}
}

@article{wang2026openclaw,
  title={Openclaw-rl: Train any agent simply by talking},
  author={Wang, Yinjie and Chen, Xuyang and Jin, Xiaolong and Wang, Mengdi and Yang, Ling},
  journal={arXiv preprint arXiv:2603.10165},
  year={2026}
}

@article{yao2024tau,
  title={{tau bench: A Benchmark for Tool-Agent-User Interaction in Real-World Domains}},
  author={Yao, Shunyu and Shinn, Noah and Razavi, Pedram and Narasimhan, Karthik},
  journal={arXiv preprint arXiv:2406.12045},
  year={2024}
}

@misc{qwen35blog,
    title = {Qwen3.5: Accelerating Productivity with Native Multimodal Agents},
    url = {https://qwen.ai/blog?id=qwen3.5},
    author = {Qwen Team},
    month = {February},
    year = {2026}
}

@online{openai_gpt_5_4,
  author = {{OpenAI}},
  title = {GPT-5.4 Thinking System Card},
  year = {2025},
  url = {https://openai.com/index/gpt-5-4-thinking-system-card/},
  urldate = {2026-04-10},
  organization = {OpenAI Deployment Safety}
}

@misc{anthropic_opus_4_6,
  author = {{Anthropic}},
  title = {{Claude Opus 4.6} Model Card},
  year = {2026},
  url = {https://www-cdn.anthropic.com/bf10f64990cfda0ba858290be7b8cc6317685f47.pdf}
}

@article{barres2025tau,
  title={{tau2 Bench: Evaluating Conversational Agents in a Dual-Control Environment}},
  author={Barres, Victor and Dong, Honghua and Ray, Soham and Si, Xujie and Narasimhan, Karthik},
  journal={arXiv preprint arXiv:2506.07982},
  year={2025}
}

@article{dong2025tool,
  title={Tool-star: Empowering llm-brained multi-tool reasoner via reinforcement learning},
  author={Dong, Guanting and Chen, Yifei and Li, Xiaoxi and Jin, Jiajie and Qian, Hongjin and Zhu, Yutao and Mao, Hangyu and Zhou, Guorui and Dou, Zhicheng and Wen, Ji-Rong},
  journal={arXiv preprint arXiv:2505.16410},
  year={2025}
}

@article{qin2025ui,
  title={Ui-tars: Pioneering automated gui interaction with native agents},
  author={Qin, Yujia and Ye, Yining and Fang, Junjie and Wang, Haoming and Liang, Shihao and Tian, Shizuo and Zhang, Junda and Li, Jiahao and Li, Yunxin and Huang, Shijue and others},
  journal={arXiv preprint arXiv:2501.12326},
  year={2025}
}

@misc{OSWorld,
      title={OSWorld: Benchmarking Multimodal Agents for Open-Ended Tasks in Real Computer Environments}, 
      author={Tianbao Xie and Danyang Zhang and Jixuan Chen and Xiaochuan Li and Siheng Zhao and Ruisheng Cao and Toh Jing Hua and Zhoujun Cheng and Dongchan Shin and Fangyu Lei and Yitao Liu and Yiheng Xu and Shuyan Zhou and Silvio Savarese and Caiming Xiong and Victor Zhong and Tao Yu},
      year={2024},
      eprint={2404.07972},
      archivePrefix={arXiv},
      primaryClass={cs.AI}
}

@article{jimenez2023swe,
  title={Swe-bench: Can language models resolve real-world github issues?},
  author={Jimenez, Carlos E and Yang, John and Wettig, Alexander and Yao, Shunyu and Pei, Kexin and Press, Ofir and Narasimhan, Karthik},
  journal={arXiv preprint arXiv:2310.06770},
  year={2023}
}

@misc{wildclawbench,
  author       = {Shuangrui Ding and Xuanlang Dai and Long Xing and Shengyuan Ding and Ziyu Liu and Jingyi Yang and Penghui Yang and Zhixiong Zhang and Xilin Wei and Yubo Ma and Haodong Duan and Jing Shao and Jiaqi Wang and Dahua Lin and Kai Chen and Yuhang Zang},
  title        = {WildClawBench},
  howpublished = {https://github.com/InternLM/WildClawBench},
  note         = {GitHub repository},
  year         = {2026}
}

@article{zhang2025landscape,
  title={The landscape of agentic reinforcement learning for llms: A survey},
  author={Zhang, Guibin and Geng, Hejia and Yu, Xiaohang and Yin, Zhenfei and Zhang, Zaibin and Tan, Zelin and Zhou, Heng and Li, Zhongzhi and Xue, Xiangyuan and Li, Yijiang and others},
  journal={arXiv preprint arXiv:2509.02547},
  year={2025}
}

@article{yi2026pivotrl,
  title={PivotRL: High Accuracy Agentic Post-Training at Low Compute Cost},
  author={Yi, Junkeun and Mosk-Aoyama, Damon and Huang, Baihe and Gala, Ritu and Wang, Charles and Devare, Sugam Dipak and Bhardwaj, Khushi and Gupta, Abhibha and Kuchaiev, Oleksii and Jiao, Jiantao and others},
  journal={arXiv preprint arXiv:2603.21383},
  year={2026}
}

@article{wu2026learn,
  title={Learn Hard Problems During RL with Reference Guided Fine-tuning},
  author={Wu, Yangzhen and Li, Shanda and Wen, Zixin and Zhou, Xin and Talwalkar, Ameet and Yang, Yiming and Huang, Wenhao and Cai, Tianle},
  journal={arXiv preprint arXiv:2603.01223},
  year={2026}
}

@inproceedings{huang2025boosting,
  title={Boosting mllm reasoning with text-debiased hint-grpo},
  author={Huang, Qihan and Dai, Weilong and Liu, Jinlong and He, Wanggui and Jiang, Hao and Song, Mingli and Chen, Jingyuan and Yao, Chang and Song, Jie},
  booktitle={Proceedings of the IEEE/CVF International Conference on Computer Vision},
  pages={4848--4857},
  year={2025}
}

@inproceedings{
dai2026harder,
title={Harder Is Better: Boosting Mathematical Reasoning via Difficulty-Aware {GRPO} and Multi-Aspect Question Reformulation},
author={Yanqi Dai and Yuxiang Ji and Xiao Zhang and Yong Wang and Xiangxiang Chu and Zhiwu Lu},
booktitle={The Fourteenth International Conference on Learning Representations},
year={2026},
url={https://openreview.net/forum?id=nfURupkdRJ}
}

@article{erdogan2025plan,
  title={Plan-and-act: Improving planning of agents for long-horizon tasks},
  author={Erdogan, Lutfi Eren and Lee, Nicholas and Kim, Sehoon and Moon, Suhong and Furuta, Hiroki and Anumanchipalli, Gopala and Keutzer, Kurt and Gholami, Amir},
  journal={arXiv preprint arXiv:2503.09572},
  year={2025}
}

@article{yang2026gui,
  title={GUI-Libra: Training Native GUI Agents to Reason and Act with Action-aware Supervision and Partially Verifiable RL},
  author={Yang, Rui and Wu, Qianhui and Wang, Zhaoyang and Chen, Hanyang and Yang, Ke and Cheng, Hao and Yao, Huaxiu and Peng, Baoling and Zhang, Huan and Gao, Jianfeng and others},
  journal={arXiv preprint arXiv:2602.22190},
  year={2026}
}

@article{deng2023mind2web,
  title={Mind2web: Towards a generalist agent for the web},
  author={Deng, Xiang and Gu, Yu and Zheng, Boyuan and Chen, Shijie and Stevens, Sam and Wang, Boshi and Sun, Huan and Su, Yu},
  journal={Advances in Neural Information Processing Systems},
  volume={36},
  pages={28091--28114},
  year={2023}
}

@inproceedings{caccia2024fine,
  title={Fine-Tuning Web Agents: It Works, But It's Trickier Than You Think},
  author={Caccia, Massimo and Thakkar, Megh and Boisvert, L{\'e}o and De Chezelles, Thibault Le Sellier and Pich{\'e}, Alexandre and Chapados, Nicolas and Drouin, Alexandre and Gasse, Maxime and Lacoste, Alexandre},
  booktitle={NeurIPS 2024 Workshop on Open-World Agents},
  year={2024}
}

@article{schulman2017proximal,
  title={Proximal policy optimization algorithms},
  author={Schulman, John and Wolski, Filip and Dhariwal, Prafulla and Radford, Alec and Klimov, Oleg},
  journal={arXiv preprint arXiv:1707.06347},
  year={2017}
}

@article{yu2025medresearcher,
  title={Medresearcher-r1: Expert-level medical deep researcher via a knowledge-informed trajectory synthesis framework},
  author={Yu, Ailing and Yao, Lan and Liu, Jingnan and Chen, Zhe and Yin, Jiajun and Wang, Yuan and Liao, Xinhao and Ye, Zhiling and Li, Ji and Yue, Yun and others},
  journal={arXiv preprint arXiv:2508.14880},
  year={2025}
}

@article{team2025tongyi,
  title={Tongyi deepresearch technical report},
  author={Team, Tongyi DeepResearch and Li, Baixuan and Zhang, Bo and Zhang, Dingchu and Huang, Fei and Li, Guangyu and Chen, Guoxin and Yin, Huifeng and Wu, Jialong and Zhou, Jingren and others},
  journal={arXiv preprint arXiv:2510.24701},
  year={2025}
}

@article{chu2026redsearcher,
  title={Redsearcher: A scalable and cost-efficient framework for long-horizon search agents},
  author={Chu, Zheng and Wang, Xiao and Hong, Jack and Fan, Huiming and Huang, Yuqi and Yang, Yue and Xu, Guohai and Zhao, Chenxiao and Xiang, Cheng and Hu, Shengchao and others},
  journal={arXiv preprint arXiv:2602.14234},
  year={2026}
}

@article{deng2025openvlthinker,
  title={Openvlthinker: Complex vision-language reasoning via iterative sft-rl cycles},
  author={Deng, Yihe and Bansal, Hritik and Yin, Fan and Peng, Nanyun and Wang, Wei and Chang, Kai-Wei},
  journal={arXiv preprint arXiv:2503.17352},
  year={2025}
}

@article{chen2025beyond,
  title={Beyond two-stage training: Cooperative sft and rl for llm reasoning},
  author={Chen, Liang and Han, Xueting and Shen, Li and Bai, Jing and Wong, Kam-Fai},
  journal={arXiv preprint arXiv:2509.06948},
  year={2025}
}

@article{chen2025sft,
  title={Sft or rl? an early investigation into training r1-like reasoning large vision-language models},
  author={Chen, Hardy and Tu, Haoqin and Wang, Fali and Liu, Hui and Tang, Xianfeng and Du, Xinya and Zhou, Yuyin and Xie, Cihang},
  journal={arXiv preprint arXiv:2504.11468},
  year={2025}
}

@article{pang2024iterative,
  title={Iterative reasoning preference optimization},
  author={Pang, Richard Y and Yuan, Weizhe and Cho, Kyunghyun and He, He and Sukhbaatar, Sainbayar and Weston, Jason},
  journal={Advances in Neural Information Processing Systems},
  volume={37},
  pages={116617--116637},
  year={2024}
}

@article{yao2026coba,
  title={CoBA-RL: Capability-Oriented Budget Allocation for Reinforcement Learning in LLMs},
  author={Yao, Zhiyuan and Zhang, Yi-Kai and Chen, Yuxin and Sun, Yueqing and Xu, Zishan and Yang, Yu and Hu, Tianhao and Gu, Qi and Su, Hui and Cai, Xunliang},
  journal={arXiv preprint arXiv:2602.03048},
  year={2026}
}

@inproceedings{li2026adacurl,
  title={Adacurl: Adaptive curriculum reinforcement learning with invalid sample mitigation and historical revisiting},
  author={Li, Renda and Huang, Hailang and Wei, Fei and Xiong, Feng and Wang, Yong and Chu, Xiangxiang},
  booktitle={Proceedings of the AAAI Conference on Artificial Intelligence},
  volume={40},
  number={27},
  pages={23123--23131},
  year={2026}
}

@article{zhai2025agentevolver,
  title={Agentevolver: Towards efficient self-evolving agent system},
  author={Zhai, Yunpeng and Tao, Shuchang and Chen, Cheng and Zou, Anni and Chen, Ziqian and Fu, Qingxu and Mai, Shinji and Yu, Li and Deng, Jiaji and Cao, Zouying and others},
  journal={arXiv preprint arXiv:2511.10395},
  year={2025}
}

@article{gu2026actor,
  title={Actor-Curator: Co-adaptive Curriculum Learning via Policy-Improvement Bandits for RL Post-Training},
  author={Gu, Zhengyao and Light, Jonathan and Astudillo, Raul and Ye, Ziyu and He, Langzhou and Zou, Henry Peng and Cheng, Wei and Paternain, Santiago and Yu, Philip S and Yue, Yisong},
  journal={arXiv preprint arXiv:2602.20532},
  year={2026}
}

@article{jiang2025vcrl,
  title={Vcrl: Variance-based curriculum reinforcement learning for large language models},
  author={Jiang, Guochao and Feng, Wenfeng and Quan, Guofeng and Hao, Chuzhan and Zhang, Yuewei and Liu, Guohua and Wang, Hao},
  journal={arXiv preprint arXiv:2509.19803},
  year={2025}
}

@article{wang2025let,
  title={Let it flow: Agentic crafting on rock and roll, building the rome model within an open agentic learning ecosystem},
  author={Wang, Weixun and Xu, XiaoXiao and An, Wanhe and Dai, Fangwen and Gao, Wei and He, Yancheng and Huang, Ju and Ji, Qiang and Jin, Hanqi and Li, Xiaoyang and others},
  journal={arXiv preprint arXiv:2512.24873},
  year={2025}
}

@article{nath2025adaptive,
  title={Adaptive guidance accelerates reinforcement learning of reasoning models},
  author={Nath, Vaskar and Lau, Elaine and Gunjal, Anisha and Sharma, Manasi and Baharte, Nikhil and Hendryx, Sean},
  journal={arXiv preprint arXiv:2506.13923},
  year={2025}
}

@article{rajeswaran2017learning,
  title={Learning complex dexterous manipulation with deep reinforcement learning and demonstrations},
  author={Rajeswaran, Aravind and Kumar, Vikash and Gupta, Abhishek and Vezzani, Giulia and Schulman, John and Todorov, Emanuel and Levine, Sergey},
  journal={arXiv preprint arXiv:1709.10087},
  year={2017}
}

@inproceedings{hester2018deep,
  title={Deep q-learning from demonstrations},
  author={Hester, Todd and Vecerik, Matej and Pietquin, Olivier and Lanctot, Marc and Schaul, Tom and Piot, Bilal and Horgan, Dan and Quan, John and Sendonaris, Andrew and Osband, Ian and others},
  booktitle={Proceedings of the AAAI conference on artificial intelligence},
  volume={32},
  number={1},
  year={2018}
}

@article{vecerik2017leveraging,
  title={Leveraging demonstrations for deep reinforcement learning on robotics problems with sparse rewards},
  author={Vecerik, Mel and Hester, Todd and Scholz, Jonathan and Wang, Fumin and Pietquin, Olivier and Piot, Bilal and Heess, Nicolas and Roth{\"o}rl, Thomas and Lampe, Thomas and Riedmiller, Martin},
  journal={arXiv preprint arXiv:1707.08817},
  year={2017}
}

@inproceedings{nair2018overcoming,
  title={Overcoming exploration in reinforcement learning with demonstrations},
  author={Nair, Ashvin and McGrew, Bob and Andrychowicz, Marcin and Zaremba, Wojciech and Abbeel, Pieter},
  booktitle={2018 IEEE international conference on robotics and automation (ICRA)},
  pages={6292--6299},
  year={2018},
  organization={IEEE}
}

@inproceedings{libardi2021guided,
  title={Guided exploration with proximal policy optimization using a single demonstration},
  author={Libardi, Gabriele and De Fabritiis, Gianni and Dittert, Sebastian},
  booktitle={International Conference on Machine Learning},
  pages={6611--6620},
  year={2021},
  organization={PMLR}
}

@article{yan2025learning,
  title={Learning to reason under off-policy guidance},
  author={Yan, Jianhao and Li, Yafu and Hu, Zican and Wang, Zhi and Cui, Ganqu and Qu, Xiaoye and Cheng, Yu and Zhang, Yue},
  journal={arXiv preprint arXiv:2504.14945},
  year={2025}
}

@article{zhang2026onpolicyrlmeetsoffpolicy,
      title={On-Policy RL Meets Off-Policy Experts: Harmonizing Supervised Fine-Tuning and Reinforcement Learning via Dynamic Weighting}, 
      author={Wenhao Zhang and Yuexiang Xie and Yuchang Sun and Yanxi Chen and Guoyin Wang and Yaliang Li and Bolin Ding and Jingren Zhou},
      year={2026},
      journal={arXiv preprint arXiv:2508.11408}, 
}

@article{ma2025learning,
  title={Learning What Reinforcement Learning Can't: Interleaved Online Fine-Tuning for Hardest Questions},
  author={Ma, Lu and Liang, Hao and Qiang, Meiyi and Tang, Lexiang and Ma, Xiaochen and Wong, Zhen Hao and Niu, Junbo and Shen, Chengyu and He, Runming and Li, Yanhao and others},
  journal={arXiv preprint arXiv:2506.07527},
  year={2025}
}

@article{liang2025squeeze,
  title={Squeeze the soaked sponge: Efficient off-policy reinforcement finetuning for large language model},
  author={Liang, Jing and Tang, Hongyao and Ma, Yi and Liu, Jinyi and Zheng, Yan and Hu, Shuyue and Bai, Lei and Hao, Jianye},
  journal={arXiv preprint arXiv:2507.06892},
  year={2025}
}

@article{fu2025srft,
  title={Srft: A single-stage method with supervised and reinforcement fine-tuning for reasoning},
  author={Fu, Yuqian and Chen, Tinghong and Chai, Jiajun and Wang, Xihuai and Tu, Songjun and Yin, Guojun and Lin, Wei and Zhang, Qichao and Zhu, Yuanheng and Zhao, Dongbin},
  journal={arXiv preprint arXiv:2506.19767},
  year={2025}
}

@article{van2018deep,
  title={Deep reinforcement learning and the deadly triad},
  author={Van Hasselt, Hado and Doron, Yotam and Strub, Florian and Hessel, Matteo and Sonnerat, Nicolas and Modayil, Joseph},
  journal={arXiv preprint arXiv:1812.02648},
  year={2018}
}

@article{zheng2025prosperity,
  title={Prosperity before Collapse: How Far Can Off-Policy RL Reach with Stale Data on LLMs?},
  author={Zheng, Haizhong and Zhao, Jiawei and Chen, Beidi},
  journal={arXiv preprint arXiv:2510.01161},
  year={2025}
}

@misc{gao2025turnsunlockinglonghorizonagentic,
      title={Beyond Ten Turns: Unlocking Long-Horizon Agentic Search with Large-Scale Asynchronous RL}, 
      author={Jiaxuan Gao and Wei Fu and Minyang Xie and Shusheng Xu and Chuyi He and Zhiyu Mei and Banghua Zhu and Yi Wu},
      year={2025},
      eprint={2508.07976},
      archivePrefix={arXiv},
      primaryClass={cs.CL},
      url={https://arxiv.org/abs/2508.07976}
}

@article{tongyidr,
  title={Tongyi DeepResearch Technical Report},
  author={Team, Tongyi DeepResearch and Li, Baixuan and Zhang, Bo and Zhang, Dingchu and Huang, Fei and Li, Guangyu and Chen, Guoxin and Yin, Huifeng and Wu, Jialong and Zhou, Jingren and others},
  journal={arXiv preprint arXiv:2510.24701},
  year={2025}
}

@article{zhou2025browsecompzh,
  title={Browsecomp-zh: Benchmarking web browsing ability of large language models in chinese},
  author={Zhou, Peilin and Leon, Bruce and Ying, Xiang and Zhang, Can and Shao, Yifan and Ye, Qichen and Chong, Dading and Jin, Zhiling and Xie, Chenxuan and Cao, Meng and others},
  journal={arXiv preprint arXiv:2504.19314},
  year={2025}
}

@article{dong2026agent,
  title={Agent-World: Scaling Real-World Environment Synthesis for Evolving General Agent Intelligence},
  author={Dong, Guanting and Lu, Junting and Huang, Junjie and Zhong, Wanjun and Liu, Longxiang and Huang, Shijue and Li, Zhenyu and Zhao, Yang and Song, Xiaoshuai and Li, Xiaoxi and others},
  journal={arXiv preprint arXiv:2604.18292},
  year={2026}
}

@article{rein2023gpqa,
  title={Gpqa: A graduate-level google-proof q\&a benchmark},
  author={Rein, David and Hou, Betty Li and Stickland, Asa Cooper and Petty, Jackson and Pang, Richard Yuanzhe and Dirani, Julien and Michael, Julian and Bowman, Samuel R},
  journal={arXiv preprint arXiv:2311.12022},
  year={2023}
}

@inproceedings{lin2022truthfulqa,
  title={Truthfulqa: Measuring how models mimic human falsehoods},
  author={Lin, Stephanie and Hilton, Jacob and Evans, Owain},
  booktitle={Proceedings of the 60th annual meeting of the association for computational linguistics (volume 1: long papers)},
  pages={3214--3252},
  year={2022}
}

@article{zhou2023instruction,
  title={Instruction-following evaluation for large language models},
  author={Zhou, Jeffrey and Lu, Tianjian and Mishra, Swaroop and Brahma, Siddhartha and Basu, Sujoy and Luan, Yi and Zhou, Denny and Hou, Le},
  journal={arXiv preprint arXiv:2311.07911},
  year={2023}
}

@article{singh2025openai,
  title={Openai gpt-5 system card},
  author={Singh, Aaditya and Fry, Adam and Perelman, Adam and Tart, Adam and Ganesh, Adi and El-Kishky, Ahmed and McLaughlin, Aidan and Low, Aiden and Ostrow, AJ and Ananthram, Akhila and others},
  journal={arXiv preprint arXiv:2601.03267},
  year={2025}
}

@article{liu2024deepseek,
  title={Deepseek-v3 technical report},
  author={Liu, Aixin and Feng, Bei and Xue, Bing and Wang, Bingxuan and Wu, Bochao and Lu, Chengda and Zhao, Chenggang and Deng, Chengqi and Zhang, Chenyu and Ruan, Chong and others},
  journal={arXiv preprint arXiv:2412.19437},
  year={2024}
}

@online{minimax_m2_1,
  author = {{MiniMax}},
  title = {MiniMax M2.1 System Card},
  year = {2025},
  url = {https://www.minimax.io/news/minimax-m21},
  urldate = {2026-04-10},
  organization = {MiniMax}
}

@article{turpin2023language,
  title={Language models don't always say what they think: Unfaithful explanations in chain-of-thought prompting},
  author={Turpin, Miles and Michael, Julian and Perez, Ethan and Bowman, Samuel},
  journal={Advances in Neural Information Processing Systems},
  volume={36},
  pages={74952--74965},
  year={2023}
}

@article{hubotter2026reinforcement,
  title={Reinforcement Learning via Self-Distillation},
  author={H{\"u}botter, Jonas and L{\"u}beck, Frederike and Behric, Lejs and Baumann, Anton and Bagatella, Marco and Marta, Daniel and Hakimi, Ido and Shenfeld, Idan and Buening, Thomas Kleine and Guestrin, Carlos and others},
  journal={arXiv preprint arXiv:2601.20802},
  year={2026}
}

@article{zhao2026self,
  title={Self-Distilled Reasoner: On-Policy Self-Distillation for Large Language Models},
  author={Zhao, Siyan and Xie, Zhihui and Liu, Mengchen and Huang, Jing and Pang, Guan and Chen, Feiyu and Grover, Aditya},
  journal={arXiv preprint arXiv:2601.18734},
  year={2026}
}

@article{ma2026skillclaw,
  title={SkillClaw: Let Skills Evolve Collectively with Agentic Evolver},
  author={Ma, Ziyu and Yang, Shidong and Ji, Yuxiang and Wang, Xucong and Wang, Yong and Hu, Yiming and Huang, Tongwen and Chu, Xiangxiang},
  journal={arXiv preprint arXiv:2604.08377},
  year={2026}
}

@inproceedings{
ji2026tree,
title={Tree Search for {LLM} Agent Reinforcement Learning},
author={Yuxiang Ji and Ziyu Ma and Yong Wang and Guanhua Chen and Xiangxiang Chu and Liaoni Wu},
booktitle={The Fourteenth International Conference on Learning Representations},
year={2026},
url={https://openreview.net/forum?id=ZpQwAFhU13}
}

@article{ji2026thinking,
  title={Thinking with Map: Reinforced Parallel Map-Augmented Agent for Geolocalization},
  author={Ji, Yuxiang and Wang, Yong and Ma, Ziyu and Hu, Yiming and Huang, Hailang and Hu, Xuecai and Chen, Guanhua and Wu, Liaoni and Chu, Xiangxiang},
  journal={ACL},
  year={2026}
}

@article{zheng2026code2world,
  title={Code2world: A gui world model via renderable code generation},
  author={Zheng, Yuhao and Zhong, Li'an and Wang, Yi and Dai, Rui and Liu, Kaikui and Chu, Xiangxiang and Lv, Linyuan and Torr, Philip and Lin, Kevin Qinghong},
  journal={arXiv preprint arXiv:2602.09856},
  year={2026}
}
